\documentclass[journal, twoside, web]{IEEEtran}

\usepackage{amsfonts,amsmath,amssymb}
\usepackage[english]{babel}
\usepackage[normalem]{ulem}
\usepackage{hyperref}

\usepackage{amsthm}
\usepackage{graphicx}
\usepackage{subcaption,caption}
\usepackage{epsfig}
\usepackage{ulem}
\usepackage{tikz}
\usepackage{mathrsfs}
\usepackage{bm}
\usetikzlibrary{shapes,arrows}
\usetikzlibrary{matrix}
\usetikzlibrary{calc,patterns,angles,quotes,arrows,automata}
\graphicspath{ {./imgs/}}

\usepackage[linesnumbered]{algorithm2e}

\newcommand{\norm}[1]{\left|\left| #1 \right|\right|}
\newcommand{\inner}[2]{\left< #1,#2 \right>}

\theoremstyle{definition}
\newtheorem{example}{Example}
\usepackage{mathabx}
\renewcommand{\check}{\widecheck}
\renewcommand{\bar}{\overline}

\newcommand{\diag}{{\rm diag}}
\newcommand{\Span}{{\rm span}}

\newcommand{\idem}{ {\rm idem}}
\newcommand{\res}{{\rm res}}
\newcommand{\supp}{{\rm supp}}

\newcommand{\one}{\bm{1}}
\newcommand{\zero}{\bm{0}}
\newcommand{\alg}{{\rm alg}}

\newcommand{\Ea}{{\mathbb{E}}} % Expectation in the algebraic framework

\newcommand{\bigequation}[1]{
\begin{figure*}[!t]
    \normalsize
    #1
    \hrulefill{}
    \vspace*{4pt}
\end{figure*}
}

\theoremstyle{theorem}
\newtheorem{problem}{Problem}
\newtheorem{lemma}{Lemma}
\newtheorem{proposition}{Proposition}
\newtheorem{conjecture}{Conjecture}
\newtheorem{theorem}{Theorem}
\newtheorem{corollary}{Corollary}
\newtheorem{definition}{Definition}

\theoremstyle{definition}

\theoremstyle{remark}
\newtheorem{remark}{Remark}

\usepackage[backend=bibtex, style=ieee, sorting=none]{biblatex}
\newcommand{\R}{{\mathbb{R}}}
\newcommand{\Prob}{{\mathbb{P}}}
\newcommand{\E}{{\mathbb{E}}}

\RestyleAlgo{ruled}
\SetKwInOut{Input}{Input}
\SetKwInOut{Output}{Output}
\SetKwInOut{Parameters}{Parameters}

\title{\LARGE \bf Algebraic Reduction of Hidden Markov Models}

\author{Tommaso Grigoletto and Francesco Ticozzi
\thanks{T. Grigoletto and F. Ticozzi are with the Department of Information Engineering, University of Padova, Via Gradenigo 6, 35131 Padova, Italy. Emails: 
\href{mailto:tommaso.grigoletto@phd.unipd.it}{\texttt{tommaso.grigoletto@phd.unipd.it}},
\href{mailto:ticozzi@dei.unipd.it}{\texttt{ticozzi@dei.unipd.it}}.}}

\date{\today}

\begin{document}

\maketitle

\begin{abstract}
The problem of reducing a Hidden Markov Model (HMM) to one of smaller dimension that exactly reproduces the same marginals is tackled by using a system-theoretic approach. Realization theory tools are extended to HMMs by leveraging suitable algebraic representations of probability spaces. We propose two algorithms that return coarse-grained equivalent HMMs obtained by stochastic projection operators: the first returns models that exactly reproduce the single-time distribution of a given output process, while in the second the full (multi-time) distribution is preserved. The reduction method exploits not only the structure of the observed output, but also its initial condition, whenever the latter is known or belongs to a given subclass. Optimal algorithms are derived for a class of  HMM, namely observable ones. 
\end{abstract}

% \begin{IEEEkeywords}

% \end{IEEEkeywords}

\section{Introduction}

Hidden Markov processes are an ubiquitous class of stochastic models that has extensive application in modeling and prediction for speech \cite{HMMspeech,HMMspeech2}, biological systems \cite{HMMbiology, HMMbiology2, HMMbiology3, vidysagarBiology}, information and communication systems \cite{HMMcomunication, HMMcomunication2, HMMcomunication3}. Dedicated optimal control and estimation methods have been developed for this class of models, see e.g.   \cite{elliott2008hidden,kumar2015stochastic,Arapostathis1993}.

In the development of the realization theory for HMMs, two related yet well distinct problems emerge: {\em constructing an HMM from data}, and {\em reducing an existing model}, when possible, to an equivalent one of smaller size. For an analysis and review of the first one, see for example \cite{vidyasagarHiddenMarkovProcesses2011,mevel2000bayesian},  and more recent results in \cite{7362037}.
In this paper, we shall focus on the reduction problem. Besides its theoretical interest, methods for model reduction are critical in effectively addressing problems in large-scale systems \cite{antoulas2005approximation, cheng2021model,sandberg2009model}. A characterization of equivalent HMMs, that is, models that produce the same output marginals of a given one, is proposed in \cite{itoIdentifiabilityHiddenMarkov1992}.  Their treatment of equivalent HMMs is based on the definition of {\em effective} spaces, which specify equivalence classes of HMMs, representing the HMM analogue of minimal realizations spaces for linear systems. In the same paper, the authors pose the problem of finding a minimal equivalent HMM. As a reduction to the effective space does not guarantee to preserve the positivity of the model, the problem has so far remained unsolved.  

In this paper, we show how effective spaces can be extended so that the reduced model remains an HMM. In fact, we propose a general approach to the model reduction problem that is based on an {\em algebraic description of probability spaces}. While this is done very frequently and almost implicitly, we take a deeper look into the algebraic structures and the associated representations. In particular, we shall need minimal algebraic models that represent a set of random variables (r.v)  and conditional expectations. Such an algebraic approach has been developed to generalize the classical Kolmogorov description to the non-commutative case so that it suitably covers quantum mechanics \cite{vonNeumann,bratteli-robinson,meyer}, but it has proven useful in many other areas,  from random matrix theory (see, e.g. the insightful introduction \cite{terrence_tao_free_probability}) to algebraic statistics \cite{schonhuth2011complete}.
In our setting, the algebraic framework and the induced matrix representations allow us to leverage on observability and reachability ideas in the characterization of equivalent models, as well as linear-algebraic algorithms that compute reduced models.
Our approach remains deeply rooted in the system-theoretic analysis of the dynamical model and can be seen as a way to construct {\em reduced stochastic realizations}  for an HMM. Furthermore, the proofs of effectiveness for the proposed methods all hinge on a result of model reduction for switched linear systems, In order to maintain the focus on HMM, the latter is presented in Appendix \ref{sec:switching}).

In what follows, we  deal with reductions of a given HMM that {\em exactly} reproduce the marginals of the original systems. This allows us to clearly illustrate the working and theoretical foundation of the method: extension to approximate reduction will be the focus of upcoming work.

Similar problems have been studied from different perspectives: in particular,  the concept of lumpability of Markov processes \cite{kemenyFiniteMarkovChains1983a}, which induces coarse-grained processes analogous to those presented here, has been employed to characterize a class of exactly reducible HMMs (2-lumpable systems), see \cite{895565} and references therein.
Other works, as \cite{ay2005reductions} and references therein, reframe the problem using cellular automata for hidden information sources and study reductions of Markov transition kernels within this abstract approach.

The differences between our approach and the existing results are manifold, both in the tools used and the nature of the results. 
In the proposed framework, we introduce and solve two types of reduction problems: preserving only the single-time marginal, or the full (multi-time) distribution of the outcomes. We show that the former, which is of interest in model reduction of master equations for statistical models or mixing processes and algorithms \cite{apers}, can lead to further reduction and smaller final models, as one might expect. 
In addition, our reductions leverage not only the structure of the measured process, but also the particular initial distribution of the HMM. We show that the initial conditions are indeed critical for obtaining minimal reductions in many situations, in particular, when the original model is initialized in an equilibrium density.
The method hinges on the use of conditional expectations as projections for obtaining a reduced representation of the dynamics. While the idea is certainly not new to the control community, see e.g. the derivation of Kalman filters \cite{kalman1961new,elliott2008hidden}, in this work we develop it in an algebraic framework. After representing a conditional expectation as a linear operator, we construct stochastic, non-square factorization of its dual with respect to the inner product associated to the expectation: the factors are then used to obtain the reduced probabilistic description, preserving its stochastic character.
Lastly, we make direct contact with system-theoretic ideas in a linear-algebraic framework, which allows for effective, practically implementable algorithms for the reduction process. In fact, while the whole analysis could be carried out in the infinite-dimensional case, we here restrict to the finite case: in order to derive computable algorithms a finite-dimensional approximation would be needed anyway.

The structure of the paper is as follows:
In Section \ref{sec:geometric_approach_to_probability_theory} we review the fundamentals of the algebraic probabilistic models needed for our aims. The approach is directly borrowed from non-commutative probability \cite{meyer, maassen} and its use in quantum theory, where the algebras used for embedding the probability space need not be commutative (and are typically infinite-dimensional \cite{bratteli-robinson}), and can then be used to model quantum systems \cite{vonNeumann}. As remarked above, in this work we only use commutative, finite-dimensional associative algebras, represented as $\R^n$ endowed with its element-wise product. Subsection \ref{sec:conditional_expectations} is focused on conditional expectations as linear maps on algebras, their duals, and their representations. These are some of the key tools in the development of our method.

Section \ref{sec:problem_definition} is devoted to introducing the notation and the problems of interest, namely obtaining reduced models that reproduce either the single-time marginals or the multi-time marginals of a given HMM, while Section \ref{sec:system_theoretic_approach} presents some preliminary results that build upon \cite{itoIdentifiabilityHiddenMarkov1992} from an explicit system-theoretic perspective. The main results of the section are obtained specializing a switched-system result that we derive in Appendix \ref{sec:switching} to maintain the focus on HMMs.
The key ideas we leverage to obtain reduced HMMs are described in \ref{sec:single-time_solution}, where a class of reduction algorithms for the single-time marginal problem is developed. Section \ref{sec:multi-time_solution} then extends and adapts these ideas to the multi-time marginal problem. 
A key point in our analysis is that, in order to develop the algorithms, we must switch from the abstract quotient spaces of \cite{itoIdentifiabilityHiddenMarkov1992} to a representative effective subspace. We show that the choice of representative has a non-trivial effect on the reduction itself.
How to select this and other parameters used in the algorithms is discussed in Section \ref{sec:optimal_choice_of_parameters}, where we provide optimal choices for a class of models that  includes observable HMMs and Markov chains. The same choices prove to be optimal in all the tested examples, also in the presence of non-observable components of the reachable space.
Some particularly instructive examples are given in Section \ref{sec:examples}, and an outlook on future developments is provided with the concluding remarks in Section \ref{sec:conclusion}.

\subsection{Basic Notation}
In the following, we typically denote vectors $\bm {v}\in \R^n$ in boldface, and matrices in capitals $V\in \R^{n\times m}.$ We denote $\one$, the vector of all ones, and $\zero$ the vector of all zeros. The matrix transpose of $V$ is $V^T$. Given a vector $\bm x\in\mathbb{R}^n$ and the standard basis $\{\bm e_i\}$ for $\R^n,$ we define its support as the vector space $\supp(\bm x) = \Span\{\bm e_i| \bm e_i^T\bm x\neq 0\}$. Given a vector space $\mathcal{V}\subseteq\mathbb{R}^n,$  its {\em support} is defined as the vector space $\supp(\mathcal{V}) = \Span\{\bm e_i| \exists \bm x\in\mathcal{V} \text{ s.t } \bm e_i^T\bm x\neq 0\}$. $\diag(\cdot)$ is the operator that, given a vector $\bm v$, $\diag(\bm v)$ returns a diagonal matrix with $[\diag(\bm v)]_{i,i} = \bm v_i$.

\section{Algebraic approach to probability theory}
\label{sec:geometric_approach_to_probability_theory}

The central idea in algebraic probability models is to represent all the key ingredients of a classical probabilistic model as elements of a suitable algebra $\mathscr{A}$, endowed with a probability functional (or state) $p$. 
In the following sections, we start from a probability space $(\Omega,\Sigma,\Prob)$ and briefly review how to construct an algebraic representation $(\mathscr{A},p),$ with $\mathscr{A}\subseteq \R^n$. Correspondingly, we show that any pair $(\mathscr{A},p)$ admits a classical representation. This allows for a natural probabilistic interpretation of the proposed reduction method.

\subsection{Fundamentals of algebraic probabilistic models}
\label{sec:random_variables}

\subsubsection{Events and $\sigma$-Algebras}
Throughout the rest of this article, we will consider finite-dimensional probability spaces $(\Omega,\Sigma,\Prob)$. Without loss of generality, we can assume $\Omega = \{1, \dots, n\}$.

The first step in the construction entails the vector representation of events. The latter are in 1-to-1 correspondence to indicator functions: let $I_E(\omega)$ be the indicator function associated with the event $E.$ Since the probability space is finite-dimensional, we can further associate indicator functions to vectors in $\R^n$. In particular, each indicator of an elementary event $\omega\in\Omega$ can be associated to its corresponding vector of the standard basis, i.e. $\bm e_\omega\in\R^n$. Similarly, we can define {\em indicator vectors} for any event $E\in\Sigma$ as $\bm f_E = \sum_{\omega\in E}\bm e_\omega$. For these vectors, $(\bm f_E)_\omega = 1$ if $\omega\in E$ and zero otherwise. Notice that $\bm f_\Omega = \one$, and $\bm f_\emptyset=\zero$.

Let us denote with $\mathcal{F}_{\Sigma}$ the set of indicator vectors of the events of the $\sigma$-algebra $\Sigma.$ Let $\wedge$ denote the element-wise product $(\bm v\wedge \bm w)_i =\bm v_i\bm w_i$,  $\vee$ denote the modified sum operation defined as $\bm v \vee \bm w = \bm v + \bm w - \bm v\wedge \bm w$ and $\neg$ denote the negation operation defined as $\neg \bm v = \one - \bm v$. 
By construction, the set $\mathcal{F}_{\Sigma}$ equipped with the operations $\wedge, \vee, \neg$ is isomorphic to the $\sigma$-algebra $\Sigma$ with $\cap, \cup, \overline{\cdot}$. 
 In the following, we  refer to $\mathcal{F}_\Sigma$ as a {\em vector $\sigma$-algebra}, and we will drop the subscript when unnecessary. 

A {\em vector partition of $\Omega$} is a subset $\mathcal{P}\subseteq\mathcal{F}\setminus\{\zero\}$ such that $\bm f_i\wedge \bm f_j = \zero$, for all $\bm f_i, \bm f_j \in\mathcal{P}$, $i\neq j$ and $\one = \vee_{\bm f_j\in \mathcal{P}} \bm f_j$.
The {\em finest resolution} 
in $\mathcal{F}$ is a partition $\res(\mathcal{F})$ such that $\bm f = \vee_{\bm f_j\in \res(\mathcal{F})} c_j \bm f_j$ with $c_j\in\{0,1\}$, for all $\bm f\in\mathcal{F}$.

Note that $\res(\mathcal{F})$ is not necessarily equal to the standard basis of $\mathbb{R}^n$ since, in general, $\Sigma$ is contained but not equal to the power set of $\Omega$. We shall also denote $\res(\Sigma)$ to indicate the finest resolution of a classical $\sigma$-algebra.

 \subsubsection{Random variables}
Random variables (r.v.)  are $\Sigma$-measurable functions $X(\omega):\Omega\to\mathbb{A}\subset\mathbb{R}$, where $\mathbb{A}=\{x_i\}$ is the finite set of outcomes of $X,$ called the alphabet. Let $E_i=X^{-1}(x_i)$. An r.v. $X$ can also be represented as linear combination of indicator function $X(\omega) = \sum_{i=0}^{|\mathbb{A}|} x_i I_{E_i}(\omega).$

Using the vector representation $\bm f_{E_i}$ of indicator functions $I_{E_i}$ in the previous equation, 
each $X$ can also be represented as a vector \[\bm x=\sum_{i=1}^{|\mathbb{A}|} x_i\bm f_{E_i}\in\R^n\] such that $\{\bm f_i\}\subset\mathcal{F}_{\Sigma}$ forms a partition of $\Omega$.
Notice that in the vector formalism, the notion of $\mathcal{F}_\Sigma$-measurability is equivalent to the condition $\bm x\in\Span\{\mathcal{F}_\Sigma\}$. Here and elsewhere, the boldface font $\bm x$ is used for (vector representations of) r.v.s, while $x$ denotes the corresponding outcome.
As we show below, $\Span\{\mathcal{F}_\Sigma\}$ has the property of being an {\em algebra}, namely a vector space (or subspace) that is closed under the element-wise product $\wedge$.  An algebra is {\em unital} if it contains $\one$.  The whole $\R^n$ is then an unital algebra, and we  denote its subalgebras using the script font, e.g. $\mathscr{A}$.  
A non-unital algebra $\mathscr{A}$ still contains the vector $\one_\mathscr{A},$ which has entries $1$ on the support of $\mathscr{A}$ and $0$ otherwise and acts as the product identity in $\mathscr{A}$.

The following proposition collects some known facts which clarify the relation between $\mathcal{F}_{\Sigma}$ and $\mathscr{A}=\Span\{\mathcal{F}_\Sigma\}$ and proves that it is indeed an algebra. 
\begin{proposition}
    \label{prop:alg}
    If $\mathcal{F}\subset \mathbb{R}^n$ is a vector $\sigma$-algebra, then $\mathscr{A} = \Span\{\mathcal{F}\}$ is the smallest subalgebra in $\mathbb{R}^n$ containing $\mathcal{F}$, and it is unital.
    Conversely, let $\mathscr{A}$ be any unital subalgebra in $\mathbb{R}^n$  and $\idem(\mathscr{A}) := \{\bm f\in\mathscr{A}|\bm f\wedge \bm f = \bm f\}\subset\mathscr{A}$ be the set of idempotent vectors in $\mathscr{A}$. Then $\idem(\mathscr{A})$ is the smallest $\sigma$-algebra such that every element in $\mathscr{A}$ is $\mathcal{F}$-measurable and $\res(\idem(\mathscr{A}))$ forms an orthogonal basis for $\mathscr{A}$.
\end{proposition}

A proof of this proposition is reported in Appendix \ref{sec:proofs} for completeness. This proposition shows that, not only does the space of $\mathcal{F}_\Sigma$-measurable random variables form an unital subalgebra, but, more importantly, given any unital subalgebra $\mathscr{A}$, it is possible to find the {\em minimal} (vector) $\sigma$-algebra that makes every random variable in $\mathscr{A}$ measurable. 
For convenience, in the following, we  refer to $\res(\idem(\mathscr{A}))$ as $\res(\mathscr{A})$.

\subsubsection{Probability and expectations}
Let now consider a probability measure $\Prob:\Omega\to[0,1]$. 
For any probability measure $\Prob[\cdot]$ on $\Sigma$ we can define a vector as follows \[\bm p := \sum_{\omega\in\Omega} \frac{\Prob[\omega]}{\inner{\bm f_\omega}{\bm f_\omega}}\bm f_\omega.\]
Then, for any $\bm f_E\in\mathcal{F}_\Sigma$ it is immediate to verify that $\Prob[E] = \inner{\bm p}{\bm f_E}$. 
In particular, notice that if we can write $\bm p := \sum_{\bm f_r\in\res(\mathscr{A})}p_{r} \bm f_r,$ we find that $\bm p$ can be interpreted as a random variable in the same algebra, $\bm p\in\mathscr{A}$. 

A vector $\bm p$ is said to be a probability vector if $\bm p_i\geq0$ for all $i$ and $\one^T\bm p = 1$. The set of probability vectors in $\mathscr{A}$ is defined as
\(\mathcal{D}(\mathscr{A}):= \{\bm p\in\mathscr{A}|\bm p_i\geq0 \quad\forall i, \quad \one^T\bm p=1\}\). Note that $\mathcal{D}(\mathscr{A}) = \mathcal{D}(\R^n)\cap\mathscr{A}.$ 

Consider a r.v. $X$ and let us denote again with $\bm f_{i}$ the indicator function associated to the outcome $x_i$. It then holds that $\Prob[X = x_i]=\inner{\bm p}{\bm f_{i}}$. 
Similarly, we can compute the expectation of a random variable as \(\E[\bm x] = \sum_j x_j \Prob[E_j]=\sum_j x_j \left<\bm p,\bm f_j\right> = \left<\bm p,\bm x\right>\). 

\vspace{2mm} In summary, we have shown that an unital subalgebra $\mathscr{A}$ can subsume both the $\sigma$-algebra and the space of measurable random variables of a given probability space. Moreover, it is equivalent to a probability space when paired with a positive linear functional, associated to the inner product with a probability vector $\bm p$. Conversely, given a pair $(\mathscr{A}, \bm p),$  we can always construct a (classical) probability space associated with the pair. This can be done by choosing $\Omega = \{1,\dots,n\}$ and the underlying $\sigma$-algebra $\Sigma$ associated to $\idem(\mathscr{A})$ as in Proposition \ref{prop:alg}. Lastly, $\bm p$ represents the probability distribution associated with the functional $\Prob[E]=\inner{\bm p}{\bm f_E}$. 

\subsection{Stochastic maps and Conditional Expectations}
\label{sec:conditional_expectations}

Let us now focus on the maps between probability vectors. 

Consider two unital subalgebras $\mathscr{F}$ of $\R^n$ and $\mathscr{G}$ of $\mathbb{R}^m$. A linear map between probability vectors $P[\cdot]:\mathcal{D}(\mathscr{F})\to\mathcal{D}(\mathscr{G})$, $\bm p\mapsto \bm q = P[\bm p]$ is called a {\em stochastic map.}
Such a map can be represented as a {\em (column)-stochastic matrix} $P\in\mathbb{R}^{m\times n}$, i.e. a matrix such that $(P)_{i,j}\geq0$ $\forall i,j$ and $\one_m^T P=\one_n^T$.     

In the following, the main task will be to find reduced descriptions of linear dynamics associated with stochastic maps. In doing this, we  exploit the properties of a particular class of stochastic maps: the duals of conditional expectations.

Recall that the conditional expectation of an r.v. given a $\sigma$-algebra $\Sigma$ with finest resolution $\res(\Sigma)$ can be written as follows:
\begin{equation}
 \E[X|\Sigma] = \sum_{E\in\res(\Sigma)}\frac{\E[I_EX]}{\E[I_E]}I_E(\omega).
 \label{eqn:classic_cond_prob_def}
 \end{equation}

Let consider a vector r.v. $\bm x\in\mathscr{F}\subseteq\mathbb{R}^n$, a unital algebra $\mathscr{A}\subseteq\mathscr{F}$ with $\{\bm a_i\}=\res(\mathscr{A})$ and $d = \dim(\mathscr{A})< n$, and the underlying probability measure $\bm p$. Following the previous definition, we can define the conditional expectation for the vector r.v. with respect to an algebra $\mathscr{A}$:
\begin{equation*}
    \E_{\bm p}[\bm x|\mathscr{A}] := \sum_{j=1}^d \frac{\inner{\bm p}{\bm x\wedge \bm a_j}}{\inner{\bm p}{\bm a_j}}  \bm a_j.
\end{equation*}
Noticing that it is a linear operator acting on $\bm x$ we can represent it as a matrix $\Ea_{|\mathscr{A},\bm p}\in\mathbb{R}^{n\times n}$, namely:
\begin{equation}
    \Ea_{|\mathscr{A},\bm p} = 
    \sum_{j=1}^d \frac{\bm a_j(\bm p\wedge \bm a_j)^T}{\inner{\bm p}{\bm a_j}}.  
    \label{eqn:cond_prob_def}
\end{equation}
Consider the inner product of the conditional expectation of $\bm x$ with a probability distribution $\bm q$, which we have shown to correspond to its expectation. The dual of the conditional expectation is then a map on the probability distribution defined as:
\begin{align*}
    \inner{\bm q}{\Ea_{|\mathscr{A},\bm p}\bm x}& =   \inner{\Ea_{|\mathscr{A},\bm p}^T\bm q}{\bm x}
\end{align*}
which gives in
\begin{equation}
    \Ea_{|\mathscr{A},\bm p}^T =  
    \sum_{j=1}^d \frac{(\bm p\wedge \bm a_j)\bm a_j^T}{\inner{\bm p}{\bm a_j}}.
    \label{eqn:dual_cond_prob_def}
\end{equation}
It is immediate to verify that $\Ea_{|\mathscr{A},\bm p}^T$ is stochastic.

The conditional expectation and its adjoint are orthogonal projectors with respect to a modified inner product. Notice that $\bm p\wedge\mathscr{A}=\Span\{\bm p\wedge \bm a_i\} = \diag(\bm p)\mathscr{A}$.

\begin{lemma}
\label{lem:cond_exp_invariance}
Let consider the modified inner product $\inner{\bm v}{\bm w}_{\bm p} = \Ea_{\bm p}[\bm v\wedge \bm w],$ with $\bm p > \bm 0$. Then $\Ea_{|\mathscr{A},\bm p}$ is the orthogonal projector onto $\mathscr{A}$ with respect to the inner product $\inner{\cdot}{\cdot}_{\bm p}$ and $\Ea_{|\mathscr{A},\bm p}^T$ is the orthogonal projector onto $\bm p\wedge\mathscr{A}$ with respect to the inner product $\inner{\cdot}{\cdot}_{\bm p^{-1}}$. 
\end{lemma}

The proof of this lemma is reported in Appendix \ref{sec:proofs} for completeness.

\begin{remark}
Note that the above Lemma also implies that $\Ea_{|\mathscr{A},\bm p}$ acts as the identity on $\mathscr{A}$ while $\Ea_{|\mathscr{A},\bm p}^T$ acts as the identity on $\bm p\wedge\mathscr{A}$.
Furthermore, they are orthogonal projections for the standard inner product $\inner{\cdot}{\cdot}$ if (and only if) $\bm p\in\mathcal{D}(\mathscr{A})$ and is positive, namely $\bm p = \sum_j \lambda_j \bm a_j\in\mathbb{R}^n$ with $\lambda_j>0$, $\sum_j\lambda_j = 1$. In this case, we have $\Ea_{|\mathscr{A},\bm p}=\Ea_{|\mathscr{A},\bm p}^T.$
\end{remark}

Consider the standard basis $\{\bm e_j\}$ for $\mathbb{R}^d,$ where $d$ is the dimension of $\mathscr{A}.$ We can then construct a (full-rank) stochastic factorization of $\Ea^T_{|\mathscr{A},\bm p}$.

\begin{proposition}
 \label{prop:stochastic_factors}
 Define
 \begin{equation}
    J = \sum_{j=1}^d \frac{(\bm p\wedge\bm a_j) \bm e_j^T}{\inner{\bm p}{\bm a_j}}\in\mathbb{R}^{n\times d},\quad
    R = \sum_{j=1}^d \bm e_j \bm a_j^T\in\mathbb{R}^{d\times n}.
    \label{eqn:exp_factorization}
 \end{equation}
Then $J,R$ are stochastic matrices that satisfy $JR = \Ea_{|\mathscr{A},\bm p}^T$, $RJ = I_d$, $\ker(R) = \mathscr{A}^\perp$ and $\ker(J^T) = (\bm p\wedge\mathscr{A})^\perp$. 
 \end{proposition}
 \begin{proof}
 $J$ and $R$ are clearly positive since both $\{\bm a_j\}$ and $\{\bm e_j\}$ are vectors of zeros and ones and $\bm p$ is positive. $J$ is clearly stochastic, $\one_n^TJ = \one_d^T$ since $\one_n^T(\bm p\wedge \bm a_j) = \inner{\bm p}{\bm a_j}$. On the other hand, we have $\one_d^TR=\sum_{j=1}^d \bm a_j^T = \one_{n}$, since $\mathscr{A}$ is unital. 
 
 We can then observe that $\bm a_j^T(\bm p\wedge\bm a_k) = \inner{\bm p}{\bm a_j\wedge\bm a_k} = \inner{\bm p}{\bm a_j}\delta_{j-k}$ to conclude that $RJ=I_d$. Finally, if we consider $\bm x\in\mathscr{A}^\perp$, i.e. $\inner{\bm x}{\bm a_j}=0$ for all $j$ we obtain $R\bm x =\zero$ and, similarly, if $\bm x\in(\bm p\wedge\mathscr{A})^\perp$, i.e. $\inner{\bm x}{\bm p\wedge\bm a_j}=0$ for all $j$ we obtain $J^T\bm x =\zero$.
 \end{proof}
 
 This stochastic factorization induces a reduction in the probabilistic description. In fact, we have that for each distribution $\bm q$ and r.v. $\bm p$:
 \begin{align*}
     \inner{\bm q}{\Ea_{|\mathscr{A},\bm p}\bm x}& =   \inner{\Ea_{|\mathscr{A},\bm p}^T\bm q}{\bm x} = \inner{JR\bm q}{\bm x}\\ 
     &= \inner{R\bm q}{J^T{\bm x}} = \inner{\check{\bm q}}{\check{\bm x}}
 \end{align*}
where we define the {\em reduced distribution}  as $\check{\bm q} := R\bm q\in\mathcal{D}(\mathbb{R}^d)$ and {reduced random variable} $\check{\bm x} := J^T\bm x \in\mathbb{R}^d$. This property shows that, given an unital algebra $\mathscr{A}$, it is possible to reduce the probabilistic description of the set of measurable events to the space $\mathbb{R}^d$ with $d=\dim(\mathscr{A})$. For this reason, we name $R$ the stochastic reduction and $J$ the stochastic injection.

 In order to obtain smaller reduced models, it is useful to notice that even if $\mathscr{A}$ is a non-unital subalgebra of $\R^n$, namely the subalgebra has limited support, we can still use the reduction via factorization. 
In particular, we can use definitions \eqref{eqn:cond_prob_def}, \eqref{eqn:dual_cond_prob_def} and \eqref{eqn:exp_factorization} to define orthogonal (for a modified product) projections on the algebra, their dual, and their factorization. We use the notation $\E_{|\mathscr{A}, \bm p}$ for simplicity, even if these are not true conditional expectations. 
One relevant difference, in this case, is highlighted in the following.

\begin{corollary}
 \label{cor:stochastic_factors_non_unital}
 Let $\mathscr{A}$ be a non-unital subalgebra and $\bm p$ be such that $\bm p_i>0$ for all $i$, then $\Ea^T_{|\mathscr{A},\bm p}$ allows for a factorization $\Ea_{|\mathscr{A},\bm p}^T = JR$ with $J$ and $R$ as defined above. Moreover, $J$ is stochastic while $R$ is stochastic over the support of $\mathscr{A}$, i.e $\one^T_d R = \one^T_{\supp(\mathscr{A})}$ and $\one^T_{\supp(\mathscr{A})}J = \one^T_d$.
 \end{corollary}
\begin{proof}
The proof is the same as \ref{prop:stochastic_factors} with the only difference that $\sum_{j=1}^d \bm a_j^T = \one_{\supp(\mathscr{A})}$ and $\one_{\supp(\mathscr{A})}^T(\bm p\wedge \bm a_j) = \inner{\bm p}{\bm a_j}$ holds, since $\mathscr{A}$ is not unital.
\end{proof}

\section{HMM and Problem definition}
\label{sec:problem_definition}

Throughout the rest of this work, we  consider stochastic processes that can be described as Markov processes or Hidden Markov processes (HMPs). 

A stochastic process $\{\bm x_t\}$ is a collection of r.v.s taking values in the finite alphabet $\mathbb{A}_{\bm x}$, indexed by time $t$. Without loss of generality, we can assume $\mathbb{A}_{\bm x}=\{1,2,\ldots,n\}$. As the alphabet is independent of time, we can choose a fixed resolution of indicator vectors $\{\bm f_{i}\}$ with respect to which $\bm x_t$ is measurable at all times, the standard basis for $\mathbb{R}^n$ being the most compact one.  With this choice, $\{\bm x_t\}$ is a sequence in $\R^n.$ 
In the following we  thus denote by $\bm x_{0:k}$ a stochastic process with $t=0,\dots,k$, with $x_{0:k}\in\mathbb{A}_{\bm x}^{k+1}$ an ordered sequence of its outcomes, i.e. $x_{0:k} = x_0, x_1, \dots , x_k$, where $x_i\in\mathbb{A}_{\bm x}$ for all $i$ and $|x_{0:k}|=k+1$. Then, the joint probability of a sequence of outcomes can be written as \(\Prob\left[\bm x_0=x_0, \dots, \bm x_k = x_k\right] = \Prob\left[\bm x_{0:k} = x_{0:k}\right]\).
A stochastic process $\{\bm x_t\}$ in $\mathbb{R}^n$ is an homogeneous Markov process if \[\Prob[\bm x_{t+1}=x_{t+1}|\bm x_{0:t} = x_{0:t}] = \Prob[\bm x_{t+1}=x_{t+1}|\bm x_t = x_t]\]
and such probability is independent of $t$ for all pairs $x_{t+1},x_{t}$. 

 In this case, we have that there exists an initial probability vector $\bm p_0\in\mathbb{R}^n$ and a stochastic matrix $P\in\R^{n\times n}$ called the {\em transition probability matrix} such that $\Prob[\bm x_0=x_0] = \inner{\bm p_0}{\bm f_{x_0}}$ and $\Prob[\bm x_{t+1}=x_{t+1}|\bm x_t = x_t] = \bm f_{x_{t+1}}^TP\bm f_{x_t}$ where $\bm f_{x_t}$ represents the elementary event associated with the outcome $x_t$.

The main focus of this work are partially-observed HMPs, better known as HMPs. The following definition adapts  \cite[Definitions 9.2 and 9.3]{vidyasagarHiddenMarkovProcesses2011} to our setting. 
\begin{definition}[Hidden Markov processes]
    A stochastic process $\{\bm y_t\}$ in $\mathbb{R}^m$ taking values in $\mathbb{A}_{\bm y}$ is an HMP if there exist a Markov process $\{\bm x_t\}$ in $\mathbb{R}^n$ taking values in $\mathbb{A}_{\bm x}$ such that $\{(\bm y_t, \bm x_t)\}$ is jointly Markov and $\Prob[\bm y_t = y_t, \bm x_t=x_t|\bm y_{t-1} = y_{t-1}, \bm x_{t-1}=x_{t-1}] = \Prob[\bm y_t = y_t, \bm x_t=x_t|\bm x_{t-1}=x_{t-1}]$ for all t. 
\end{definition}   
    For HMPs, there exists an initial probability distribution $\bm p_0$ and transition probability matrix $P\in\mathbb{R}^{n\times n}$ defined as before, as well as a stochastic matrix $C\in\mathbb{R}^{m\times n},$ called {\em emission probability matrix}, such that $\Prob[\bm y_t=y_t|\bm x_t=x_t] = \bm e_{y_t}^TC\bm f_{x_t}$, where $\{\bm e_i\}$ is the standard basis for $\mathbb{R}^m,$ and $\bm e_{y_t}$ represents the elementary event associated to  $y_t$. 
\begin{definition}[Hidden Markov models]
    We define a {\em Hidden Markov Model (HMM)} as the couple $\theta = (P, C)$. 
\end{definition}

 The HMM $\theta$ and the initial distribution $\bm p_0$ completely characterize the evolution of the probability distributions, leaving $n,m$ and the alphabets implicit. In fact, the marginal distribution evolution can be modeled by
\begin{equation}
    \begin{cases}
        \bm p(t+1) = P \bm p(t)\\
        \bm q(t) = C\bm p(t)
    \end{cases}
    \label{eqn:single_time_marginal_model}
\end{equation}
associated to $\theta$ and initial condition $\bm p(0) = \bm p_0$ and can then be computed as 
\[\Prob_{\theta,{\bm p_0}}[\bm y_t = y_t]=\bm e_{y_t}^TCP^t\bm p_0.\] Notice that we made the dependence on the HMM $\theta$ and initial distribution $\bm p_0$ explicit whenever necessary to distinguish distributions induced by different models.

We are ready to state the first of the problems we will address in the following sections.

\begin{problem}[Single-time marginals]
    \label{prb:single_time}
    Given an HMM $\theta = (P, C)$ and a finite set of initial probability distributions $\mathcal{S}\subset\mathcal{D}(\mathbb{R}^n)$ find a reduced HMM $\check{\theta} = (\check{P}, \check{C})$ of dimension $d\leq n$ and a linear map $\Psi[\cdot]:\mathcal{S}\to\mathcal{D}(\mathbb{R}^d)$, $\bm p_0\mapsto\check{\bm p}_0$ such that 
    $$\Prob_{\theta,{\bm p_0}}[\bm y_t = y_t] = \Prob_{\check{\theta},{\Psi[\bm p_0]}}[\bm y_t = y_t]$$ 
    for all $t\geq0$ and for any initial conditions $\bm p_0 \in\mathcal{S}$. 
\end{problem}

The second problem that we address targets multi-time probability distributions.

\begin{problem}[Multi-time marginals]
\label{prb:multi_time}
    Given an HMM $\theta = (P, C)$ and a finite set of initial probability distributions $\mathcal{S}\subset\mathcal{D}(\mathbb{R}^n)$ find a reduced HMM $\check{\theta} = (\check{P}, \check{C})$ of dimension $d\leq n$ and a linear map $\Psi[\cdot]:\mathcal{S}\to\mathcal{D}(\mathbb{R}^d)$, $\bm p_0\mapsto\check{\bm p}_0$ such that 
    $$\Prob_{\theta,{\bm p_0}}[\bm y_{0:k} = y_{0:k}] = \Prob_{\check{\theta},{\Psi[\bm p_0]}}[\bm y_{0:k} = y_{0:k}]$$ for all sequences of the output process $y_{0:k}$ and for all initial conditions $\bm p_0 \in\mathcal{S}$. 
\end{problem}

\begin{remark}
Although Problem \ref{prb:multi_time} is more natural than Problem \ref{prb:single_time} for the typical HMM setting, the latter is also interesting in particular cases, which include efficiently simulating an unmeasured stochastic evolution, and reproducing the mixing properties of lifted chains with more compact models. In fact, while we derive solutions of Problem 2 that are also solutions for Problem 1, the size of the effective multi-time reduced model is going to be in general significantly larger, as it must exactly reproduce all transition probabilities -- see also Proposition \ref{prop:spaces_inclusions} below.
\end{remark}

\begin{remark}
As we pointed out before, in Problems \ref{prb:single_time} and \ref{prb:multi_time} we have assumed that $\mathcal{S}$ is a finite set. This assumption can be relaxed since, as we  show below, the proposed solution works for any initial condition contained in $\Span\{\mathcal{S}\}$. For this reason, when dealing with linear spaces of initial conditions one can study the problem where $\mathcal{S}$ are the generators of the set.  
\end{remark}

\section{Preliminary results:\\ A system theoretic viewpoint}
\label{sec:system_theoretic_approach}

Finding minimal realization of linear systems has been a central problem in control and system theory, for which well-established solutions are available. Nonetheless, when positivity is required on the reduced model, the minimal realization problem is, to the best of our knowledge, still open. In this section, we  review some existing results, and extend and adapt them so that they can be used in our scenarios. In particular, we shall allow for non-minimal realizations in order to guarantee their positivity.

\subsection{Single time-marginal problem}
\label{sec:single-time_marginal_intro}
Let us start by considering model \eqref{eqn:single_time_marginal_model} with initial condition $\bm p_0\in\mathcal{S}$. 
Let us define the {\em non-observable subspace}  as:
\begin{equation}
    \mathcal{N} := \ker\begin{bmatrix} C\\CP\\\vdots\\CP^{n-1}
\end{bmatrix}.
\label{eqn:non_observable_space_definition}
\end{equation}
The subspace $\mathcal{N}$ can be characterized as the largest $P$-invariant subspace contained in $\ker C$ \cite{marroControlledConditionedInvariants1994, wonhamLinearMultivariableControl}. In the case of HMM the non-observable subspace has another useful property.

\begin{lemma}
\label{lem:non_obs_orth_to_one}
For all $\bm x\in\mathcal{N}$ it holds $\one^T\bm x=0$.
\end{lemma}
\begin{proof}
    From the definition of non-observable space, we have that $\bm x\in\mathcal{N}$ if and only if $CP^t\bm x = \zero$ for all $t\geq 0$. If we then left-multiply by $\one^T$ on both sides we obtain $\one^TCP^t\bm x = \one^TP^t\bm x = \one^T\bm x = \one^T\zero = 0$ for all $\bm x\in\mathcal{N}$. 
\end{proof}

Next, define $\mathcal{R}$ as the smallest linear space that contains all probability distributions $\bm p(t)$ generated by the HMM for every $t\geq0$ and any initial distribution $\bm p_0\in\mathcal{S}$:
\begin{equation}
    \mathcal{R} := \Span\{P^t\bm p_0| t\geq 0,\bm p_0\in\mathcal{S}\}.
    \label{eqn:reachable_space_definition}
\end{equation}

\begin{remark}
The space $\mathcal{R}$ is, in fact, the reachable subspace of a state-space model in the typical form:
\begin{equation}
    \begin{cases}
        \tilde{\bm p}(t+1) = P \tilde{\bm p}(t) + B \bm u(t)\\
        \bm q(t) = C \tilde{\bm p}(t)
    \end{cases}
    \label{eqn:input_model}
\end{equation}
where $B\in\mathbb{R}^{n\times |\mathcal{S}|}$ is a matrix whose columns are the initial conditions in $\mathcal{S}$. This model reproduces the trajectories of \eqref{eqn:single_time_marginal_model} for inputs corresponding to discrete impulses. 
The non-observable subspaces of \eqref{eqn:single_time_marginal_model} and \eqref{eqn:input_model} are the same, and  the subspace $\mathcal{R}$ coincides with the reachable subspace of model \eqref{eqn:input_model}, and thus shares the same properties: $\mathcal{R}$ is the smallest $P$-invariant subspace that contains $\Span\{\mathcal{S}\}$. In light of this, we  call the $\mathcal{R}$ defined above the {\em reachable subspace.}
\end{remark}

Lastly, we call {\em effective subspace} $\mathcal{E}$ any subspace 
\begin{equation}
    \mathcal{E}\subseteq\mathbb{R}^n \text{ such that } (\mathcal{R}\cap\mathcal{N})\oplus\mathcal{E} = \mathcal{R}
\end{equation}
namely, a completion of the intersection $\mathcal{R}\cap\mathcal{N}$ to the reachable subspace $\mathcal{R}$. Notice that 
the choice of $\mathcal{E}$ is not unique, in fact, any representative of the quotient space $\mathcal{R}/(\mathcal{R}\cap\mathcal{N})$ is a suitable candidate for this choice. The most natural choice for the effective subspace is of course the orthogonal complement (with respect to the natural inner product) of $\mathcal{R}\cap\mathcal{N}$ in $\mathcal{R}$, which we shall denote with $\mathcal{E}_\perp$. Any other orthogonal complement, with respect to a modified inner product, would also be a suitable choice for $\mathcal{E}$. 

\begin{remark} The situation is reminiscent of the classical linear state-space analysis proposed by Rosenbrock \cite{rosenbrock1970state}, where all representatives of the quotient space $\mathcal{R}/(\mathcal{R}\cap\mathcal{N})$ are equivalent and associated to {\em  minimal realizations}. In our case, however, ${\cal E}$ needs to be further extended to ensure positivity of the reduced dynamical matrix, a notion that depends on the chosen reference basis. For this reason, not all choices of the effective subspace are equivalent. While we will show how the algorithm we propose works with any choice of the effective subspace,  in Section \ref{sec:optimal_choice_of_parameters} we will argue that the choice of the representative $\mathcal{E}$ of $\mathcal{R}/(\mathcal{R}\cap\mathcal{N})$ plays a key role in constructing an optimal reduction.\end{remark}

As we just recalled,  the restriction of model \eqref{eqn:input_model} to (any) $\mathcal{E}$ corresponds to a minimal realization (yet not necessarily positive or stochastic).
The next corollary shows that the same reduction method can be used for the HMM \eqref{eqn:single_time_marginal_model}, while also allowing for extensions of the effective space. In this case, the minimality of the linear realization may be lost, but will later allow us to enforce positivity. The proof relies on a related result for general autonomous switching systems we present in detail in Appendix \ref{sec:switching}.

\begin{corollary}
\label{cor:single-time_model_reduction}
Consider an effective subspace $\mathcal{E}$ for the HMM \eqref{eqn:single_time_marginal_model} and a subspace $\mathcal{V}$ such that $\mathcal{E}\subseteq\mathcal{V}$ with $d=\dim(\mathcal{V}).$  Let $\Pi_\mathcal{V}$ be the orthogonal projection onto $\mathcal{V}$ with respect to an arbitrary inner product $\inner{\cdot}{\cdot}$, such that $\Pi_\mathcal{V}(\mathcal{R}\cap\mathcal{N})\subseteq\mathcal{R}\cap\mathcal{N}$. Let $R:\mathbb{R}^n\to\mathbb{R}^d$ and $J:\mathbb{R}^d\to\mathcal{V}$ be two (non-square) factors of the orthogonal projection, $\Pi_\mathcal{V} = JR$. 

Define the reduced model $(\check{P}, \check{C}) = (RPJ,CJ)$ and the map $\check{\bm p}_0 = R\bm p_0$, for all $\bm p_0\in\mathcal{S}$. Then the linear systems associated with the pairs $(P,C)$ and $(\check{P}, \check{C})$ reproduce the same marginal distribution at a specific time instant, i.e. \[CP^t\bm p_0 = \check{C}\check{P}^t\check{\bm p}_0\] for all $t\geq 0$ and any initial condition $\bm p_0\in\Span\{\mathcal{S}\}$.
\end{corollary}
\begin{proof}
    This result follows from the application of Theorem \ref{thm:general_model_reduction} reported in Appendix \ref{sec:switching} with only one $F_i=P$, $H=C$ and $\bm x(0)=\bm p_0$. 
\end{proof}
In the following sections, we shall construct $\mathcal{V}$ so that the reduction is also an HMM.

\subsection{Multi-time marginal problem}
\label{sec:single-time_marginal_intro}

For the multi-time marginal problem, following on the seminal work \cite{itoIdentifiabilityHiddenMarkov1992}, we will consider $C$ for our initial model to have only zero or one entries, i.e. $C\in\{0,1\}^{m\times n}$. The assumption is not restrictive, as any Hidden Markov Process admits a realization with $C$ of this type  \cite[Theorem 9.4]{vidyasagarHiddenMarkovProcesses2011}.

The minimal reduction of the system producing the multi-time distribution can be obtained along the same lines.
Calculating the probability of a sequence of events is however more involved: \cite[Lemma 1]{itoIdentifiabilityHiddenMarkov1992} provides a closed form for such a computation. 
We report it here for completeness.

\begin{lemma}
    Given an HMM $\theta$ and an initial probability distribution $\bm p_0$, the probability of a sequence of outcomes is given by
    \[\Prob_{\theta,{\bm p_0}}[\bm y_{0:k} = y_{0:k}] = \one^T P_C^{y_{0:k}}  \bm p_0\] where
    \[P_C^{y_{0:k}} = P_C^{y_{1:k}} \diag(\bm e_{y_0}^T C), \]
    \[ P_C^{y_{1:k}} = \prod_{i=k}^1 P_C^{y_i}, \quad P_C^{y_i} = \diag(\bm e_{y_i}^T C)P, \quad i>0. \]

\end{lemma}
In the above lemma, the multiplication by the diagonal matrices $\diag(\bm e_{y_t}^T C)$ accounts for the conditioning of $\bm p_t$ on the outcome $\bm y_t=y_t.$ Without the latter we obtain the formulas for the single marginals.

In order to exploit system-theoretic tools, it is useful to write the probability of a sequence of outcomes as the output of a dynamical model. The dynamical model we are going to present next resembles the ``observables representations of HMMs'' described in \cite{HSU20121460}.
Call $\psi(t)=\mathbb{P}(\bm y_{0:t}=y_{0:t}).$ We can obtain its evolution as the output of a discrete-time, autonomous, switching, linear system described by
\begin{equation}
    \begin{cases}
        \bm \phi(t+1) = P_C^{y_t} \bm \phi(t)\\
         \bm \psi(t) = \one^T\bm \phi(t)
    \end{cases}
    \label{eqn:multi_time_model}
\end{equation}
with initial condition $\bm \phi_{y_0}(1) = \diag(\bm e_{y_0}^TC) \bm p_0$, $P_C^{y_t}$ defined as in the previous lemma and where $\bm  \psi(t)$ represents the probability associated to the sequence of events $y_{0:t}$. Clearly, the output $\bm \psi(t)$ depends on the sequence of $P_C^{y_t},$ which in turn depends on the outcomes of the sequence.  The output at any time $k>0$ can be computed as \(\bm \psi({y_{0:k}}) = \one^T \prod_{i=k}^1 P_C^{y_i} \bm \phi_{y_0}(1)\), while for $l=0$ we have \(\bm \psi({y_0}) = \one^T\bm \phi_{y_0}(1)\), thus recovering the formulas of the lemma. 

Given a finite set $\mathcal{S}$ of initial distributions of interest, the corresponding set of initial conditions for this model is $\Phi = \bigcup_{y_0}\diag(\bm e_{y_0}^T C)\mathcal{S}$.

Following the approach of \cite{itoIdentifiabilityHiddenMarkov1992} in a system-theoretic setting, we can define the reachable, non-observable, and effective subspaces for the multi-time problem. To avoid confusion with the previous definitions, we  call these the {\em conditioned} subspaces and denote them with a $\mathcal{C}$ subscript. Given an HMM $(P,C)$ and a set of initial conditions $\mathcal{S}$ we define the {\em conditioned non-observable subspace} as: 
\begin{equation}
    \mathcal{N_C} := \{\bm v\in\mathbb{R}^n | \one^TP_C^{y_{0:l}}\bm v=0, \quad \forall y_{0:l}\},
\end{equation}
and the {\em conditioned reachable subspace} as
\begin{equation}
    \mathcal{R_C}:= \Span\{P_C^{y_{0:l}}\bm p_0, \quad \forall {y_{0:l}}, \quad \forall\bm p_0\in\mathcal{S} \}.
\end{equation}
We can then define the {\em conditioned effective subspace} $\mathcal{E_C}$ as a completion of the intersection $\mathcal{R_C}\cap\mathcal{N_C}$ to the conditioned reachable subspace $\mathcal{R_C}$, i.e. $\mathcal{E_C}\oplus(\mathcal{R_C}\cap\mathcal{N_C})= \mathcal{R_C}$. As before, the choice of $\mathcal{E_C}$ is not unique, as any representative of the quotient space $\mathcal{R_C}/(\mathcal{R_C}\cap\mathcal{N_C})$ is a suitable choice.

The properties of these spaces have been described in \cite[Lemma 3, Section 3]{itoIdentifiabilityHiddenMarkov1992}. We  recap them in the following Lemma for the reader's convenience. 
\begin{lemma}
\label{lem:conditional_subspaces_invariance}
$\mathcal{N_C}$ and $\mathcal{R_C}$ are $P$-invariant, $\diag(\bm e_i^TC)$-invariant for all $i$ and thus, $P_C^{y_{0:l}}$-invariant for all sequences $y_{0:l}$. 

A result similar to Cayley-Hamilton Theorem holds and lets us compute the spaces by using a finite number of generators:
\begin{align} 
    \mathcal{N_C} &= \{\bm v\in\mathbb{R}^n | \one^TP_C^{y_{0:l}}\bm v=0, \quad \forall {y_{0:l}} \text{ s.t. } l<n\},
    \label{eqn:cond_non_observable_definition} \\
    \mathcal{R_C} &= \Span\{P_C^{y_{0:l}}\bm p_0, \quad \forall\bm p_0\in\mathcal{S}, \quad  \forall {y_{0:l}}  \text{ s.t. } l<n  \}.\label{eqn:cond_reachable_definition}
\end{align}
\end{lemma}

We can then notice that $\mathcal{N_C}$ is the non-observable subspace of model \eqref{eqn:multi_time_model} see e.g. \cite{ge2001reachability}, $\mathcal{R_C}$ is its reachable subspace and $\mathcal{E_C}$ is its effective subspace. The second statement holds trivially, while the first holds because $\mathcal{N_C}$ is $\diag(\bm e_i^TC)$-invariant for all $i$. The third follows by combining the first two.

An useful property of the propagator $P_C^{y_{0:k}}$ is proved in the following Lemma.
\begin{lemma}
\label{lem:str_sum}
The sum over all sequences $y_{0:k}$ of the same length $k$ of $P_C^{y_{0:k}}$ is equal to the $k$-th power of P, i.e.
 $$ \sum_{y_{0:k}} P_C^{y_{0:k}} = P^{k}$$
\end{lemma}
\begin{proof}
The statement is simply proved by observing that $\sum_{y_i} \diag(\bm e_{y_i} C) = I$ for all $i$ and summing over all the possible strings ${y_{0:k}}$, starting from the first character.
\end{proof}
 
The next Proposition shows that, in general, solving the multi-time marginal case requires a larger model than the single-time case defined before. 
\begin{proposition}
\label{prop:spaces_inclusions}
It holds that \[\ker C\supseteq \mathcal{N} \supseteq \mathcal{N_C},\]
\[\mathcal{S} \subseteq \mathcal{R} \subseteq \mathcal{R_C},\]
and also \[\mathcal{E}\subseteq\mathcal{E_C}.\]
\end{proposition}
The proof of this Lemma can be found in Appendix \ref{sec:proofs}.

\begin{remark}
This result clarifies the relation as well as the distinction between problems \ref{prb:multi_time} and \ref{prb:single_time}. In fact, this Proposition shows that, at least in principle, there could be a larger reduction if we are only interested in describing only the evolution of the marginal distribution at a specific time. Moreover, the conditioned effective subspace contains the effective subspace, thus showing, due to Corollary \ref{cor:single-time_model_reduction}, that a solution for Problem \ref{prb:multi_time} is also a solution for Problem \ref{prb:single_time}. 
\end{remark}

We now propose a class of effective model reductions for the multi-time marginal problem.
\begin{corollary}
\label{cor:multi-time_model_reduction}
Consider any conditioned effective subspace $\mathcal{E}_{\cal C}$ and  subspace $\mathcal{V}$ such that $\mathcal{E}_{\cal C}\subseteq\mathcal{V}$ with $d=\dim(\mathcal{V})$, and let $\Pi_\mathcal{V}$  be the orthogonal projection onto $\mathcal{V}$ with respect to an inner product $\inner{\cdot}{\cdot}$, such that $\Pi_\mathcal{V}(\mathcal{R_C}\cap\mathcal{N_C})\subseteq \mathcal{R_C}\cap\mathcal{N_C}$. Let $R:\mathbb{R}^n\to\mathbb{R}^d$ and $J:\mathbb{R}^d\to\mathcal{V}$ be two (non-square) factors of the orthogonal projection, $\Pi_\mathcal{V} = JR$. 

Let then consider the reduced model $(\{\check{P}_C^{y_i}\}, \one_m^T) = (\{RP_C^{y_i}J\},\one_n^TJ)$ and the map $\bm \check{\phi}(1) = R \bm \phi(1)$ for all $\bm \phi(1)\in\Phi$. Then the two models described by equations \eqref{eqn:multi_time_model} and denoted by the couples $(\{P_C^{y_i}\}, \one_n^T) $ and $(\{\check{P}_C^{y_i}\}, \one_m^T)$ reproduce the same probability of a sequence of outcomes, i.e. \[\one_n^T\prod_{j=k}^0P_C^{y_j}\bm \phi(1) = \one_m^T\prod_{j=k}^0\check{P}_C^{y_j}\check{\bm \phi}(1)\] for any sequence $y_{0:k}$ and any initial condition $\bm \phi(1)\in\Span\{\Phi\}$.
\end{corollary}
\begin{proof}
    This result follows from the application of Theorem \ref{thm:general_model_reduction} reported in Appendix \ref{sec:switching} with $F_i = \diag(\bm e_{y_i}^TC)P$, $H = \one^T$, $\bm x(0)=\diag(\bm e_{y_0}^TC)\bm p_0$.
\end{proof}

\begin{remark}
At this point one may notice that Corollary \ref{cor:multi-time_model_reduction} provides a reduction for model \eqref{eqn:multi_time_model} which includes the conditioning as part of the dynamics and in general may not translate directly into a reduction of \eqref{eqn:single_time_marginal_model} in the HMM form $(\check{P},\check{C},\check{\mathcal{S}})$. Nevertheless, we anticipate here that the algorithm  we propose in Section \ref{sec:multi-time_solution} for the multi-time case provides a model in HMM form, thanks to Lemma \ref{lem:conditional_subspaces_invariance}. Thanks to Proposition \ref{prop:spaces_inclusions} and Corollary \ref{cor:single-time_model_reduction} the obtained model also reproduces  the single-time marginals.
\end{remark}

\begin{remark}
    The two main results in this section, Corollary \ref{cor:single-time_model_reduction} and \ref{cor:multi-time_model_reduction}, as well as the underlying Theorem \ref{thm:general_model_reduction} shown in Appendix \ref{sec:switching}, have been stated for time-invariant dynamics for sake of simplicity. While it is possible to generalize the analysis to time-dependent systems,  in that case, Cayley-Hamilton-type results do not apply and consequently, the computation of reachable and non-observable spaces may become impractical.
\end{remark}

\section{Single-time solution}
\label{sec:single-time_solution}
In this section, we illustrate how to obtain solutions to Problem \ref{prb:single_time} appropriately choosing $\mathcal{V}$ in Corollary \ref{cor:single-time_model_reduction}. We  first discuss the intuition behind the method, next we present the proposed solution in form of a parametric algorithm, and  prove that, under appropriate constraints, the algorithm indeed provides a solution. Finally, in Section \ref{sec:optimal_choice_of_parameters}, propose a way to choose the relevant parameters.  

\subsection{Intuition}

The core idea behind the method stems from the fact that in order to define an HMM we need an underlying probability space and, as we have seen in Section \ref{sec:geometric_approach_to_probability_theory}, any probability space is associated to an algebra. This directly suggests that, in order to preserve the (stochastic) HMM structure in the reduction it is natural to restrict the model to an algebra whose dual contains the effective subspace, and then use the dual of the conditional expectation to obtain a stochastic reduction.

More in detail, consider the two stochastic reduction matrices $R$ and $J$  obtained in Section \ref{sec:conditional_expectations} as factors of the dual of a conditional expectation $\E^T_{|\mathscr{A}, {\bm p}},$ which is an orthogonal projection onto $\bm p\wedge\mathscr{A}$ with respect to the inner product $\inner{\cdot}{\cdot}_{\bm p^{-1}}$. Then, according to Corollary \ref{cor:single-time_model_reduction} we know that as long as  $\mathcal{E}\subseteq \mathcal{V}=\bm p\wedge\mathscr{A}$, and $\E^T_{|\mathscr{A}, {\bm p}}$ leaves $\mathcal{R}\cap\mathcal{N}$ invariant, then the reduced model reproduces the same marginal distribution as the original one. 

In order to choose $\mathscr{A}$ such that $\mathcal{E} \subseteq \bm p\wedge\mathscr{A}$ we can $\wedge$-multiply left and right by $\bm p^{-1}$ obtaining $\bm p^{-1}\wedge\mathcal{E}\subseteq \mathscr{A}$. Let $\alg(\mathcal{X})$ denote the minimal sub-algebra of $\R^n$ containing the set $\mathcal{X}.$ Then, if we define $\mathscr{A}:= \alg(\bm p^{-1}\wedge\mathcal{E})$, we ensure that $\mathcal{E}\subseteq \bm p\wedge\mathscr{A}$ is satisfied and that the reduced model reproduces the same marginal at a single time.   

To make this idea more concrete, we provide a simple illustrative example, which also highlights the importance of choosing the distribution $\bm p$ to be used in $\E^T_{|\mathscr{A}, {\bm p}}$. 

\begin{example}
Let us consider the following HMM:
\begin{align*}
    P &= \begin{bmatrix}2/5&0&1/5\\0&2/5&1/5\\3/5&3/5&3/5\end{bmatrix}, \quad \mathcal{S} = \left\{\begin{bmatrix}1/5\\1/5\\3/5\end{bmatrix}\right\}\\
C &= \begin{bmatrix}1&1&0\\0&0&1\end{bmatrix}.
\end{align*} 
Notice that $\bm p_0$ is an equilibrium, $P\bm p_0=\bm p_0$ thus the output distribution is equal to $\bm q(t) = \begin{bmatrix} 2/5& 3/5 \end{bmatrix}^T$, $\forall t\geq0$.
We can then compute the following.
\begin{align*}
    \mathcal{R}=\Span\left\{ \begin{bmatrix}1/5\\1/5\\3/5\end{bmatrix} \right\}, \quad
    \mathcal{N}=\Span\left\{ \begin{bmatrix}1\\-1\\0\end{bmatrix} \right\}
\end{align*}
and $\mathcal{R}\cap\mathcal{N}=\Span\{\zero\}$ and we can thus choose $\mathcal{E}=\mathcal{R}$. 
If we then choose $\bm p = \one$ we obtain  $$\mathscr{A}=\alg(\mathcal{R})=\Span\left\{\begin{bmatrix} 1\\1\\0 \end{bmatrix}, \begin{bmatrix} 0\\0\\1 \end{bmatrix}\right\}$$
and thus the relative factors of the dual of the conditional expectation are 
$$ R = \begin{bmatrix}1&1&0\\0&0&1\end{bmatrix}, \quad J = \begin{bmatrix}1/2&0\\1/2&0\\0&1\end{bmatrix}$$ and the associated reduced HMM is 
$$ \check{P} = \begin{bmatrix}2/5&2/5\\3/5&3/5\end{bmatrix}, \quad 
\check{C} = \begin{bmatrix}1&0\\0&1\end{bmatrix}, \quad 
\check{\bm p}_0 = \begin{bmatrix}2/5\\3/5\end{bmatrix}$$
which correctly reproduces the output marginal distribution $\bm q(t) = \begin{bmatrix} 2/5& 3/5 \end{bmatrix}^T$, $\forall t\geq0$.

On the other hand, if we were to choose $\bm p=\bm p_0$ we would obtain a different result. In fact, in that case, we have 
$$\mathscr{A}=\alg(\bm p^{-1}\wedge \mathcal{R}) = \Span\left\{\begin{bmatrix} 1\\1\\1 \end{bmatrix}\right\}$$
and thus the relative factors of the dual of the conditional expectation are  
$R = \begin{bmatrix}1&1&1\end{bmatrix}$, and $J = \begin{bmatrix}1/5&1/5&3/5\end{bmatrix}^T$ and the associated reduced HMM is
$\check{P}=1$, $\check{C}=\begin{bmatrix} 2/5& 3/5 \end{bmatrix}$ and $\bm p_0 = 1$ which also reproduces the output marginal distribution and is clearly minimal (optimal reduction). This shows that the choice of $\bm p$ is important if we are interested in minimizing the dimension of the reduced model.
\end{example}

\subsection{Proposed solution}

We now formalize the proposed method to solve Problem \ref{prb:single_time} in the following Algorithm. Let $\Gamma(\mathcal{R},\mathcal{N})$ be a map that selects an effective space $\cal{E}$ given some $\mathcal{R},\mathcal{N}.$

\begin{algorithm}
    \caption{HMM reduction for problem \ref{prb:single_time}}
    \label{algo:model_reduction}
    \SetAlgoLined
    \Input{$(P,C)$, $\mathcal{S}$.}
    \Parameters{$\bm p$, $\Gamma$.}
    Compute $\mathcal{R}$ and $\mathcal{N}$ using equations \eqref{eqn:reachable_space_definition} and \eqref{eqn:non_observable_space_definition}\;
    Compute $\mathcal{E} = \Gamma(\mathcal{R},\mathcal{N})$\;
    Compute $\mathscr{A}:=\alg({\bm p}^{-1}\wedge\mathcal{E})$\;
    Compute $\mathbb{E}_{|\mathscr{A}, {\bm p}}^T$ using equation \eqref{eqn:dual_cond_prob_def} \;
    If $\E^T_{|\mathscr{A},\bm p}(\mathcal{R}\cap\mathcal{N})\nsubseteq\mathcal{R}\cap\mathcal{N}$:
    redefine $\mathscr{A}:=\alg({\bm p}^{-1}\wedge\mathcal{R})$ and recompute $\mathbb{E}_{|\mathscr{A}, {\bm p}}^T$ \;
    Compute the factors $R$ and $J$ of $\mathbb{E}_{|\mathscr{A}, {\bm p}}^T$ with the definition given in equation \eqref{eqn:exp_factorization}\;
    \Output{ $(\check{P}, \check{C}) = (RPJ,CJ)$ and $R$.}
\end{algorithm}

Notice that this algorithm depends, in addition to its inputs, on two parameters: the first one, $\bm p$, is a positive vector; the second one, is the map $\Gamma$ that selects the effective subspace. We will discuss more in detail the choice of the effective subspace in Section \ref{sec:optimal_choice_of_parameters}.

We are finally ready to prove that Algorithm \ref{algo:model_reduction} solves the single-time marginal problem.

\begin{theorem}
\label{thm:single_time_algo}
For any choice of $\mathcal{E}$ and $\bm p$ positive i.e. ${\bm p_i}>0$ $\forall i$, Algorithm \ref{algo:model_reduction} provides a solution to Problem \ref{prb:single_time}.
\end{theorem}
\begin{proof}
To prove the statement we have to prove that: i) The reduced model $\check{\theta}=(\check{P}, \check{C})$ and the linear map $R$ provide the same marginal distribution at any time as the original model; ii) the reduced model $\check{\theta}$ is an HMM, and $R\bm p_0$ is a probability vector.

We shall start by proving the first point. We do so leveraging Corollary \ref{cor:single-time_model_reduction}. First of all, we have that, for any vector $\bm p$ such that $\bm p_i>0$ for all $i$ the inner product $\inner{\cdot}{\cdot}_{\bm p}$ is positive-definite and thus well defined. Moreover, by definition of the algebra $\mathscr{A}$, we have that, for any choice of the effective subspace $\mathcal{E}$ it holds $\mathcal{E} \subseteq \bm p\wedge \mathscr{A}$ so, by choosing $\mathcal{V} = \bm p\wedge\mathscr{A}$, and using the restriction and injection map defined in equation \eqref{eqn:exp_factorization},  i) follows from Corollary \ref{cor:single-time_model_reduction} if case $\E^T_{\mathscr{A},\bm p}(\mathcal{R}\cap\mathcal{N})\subseteq\mathcal{R}\cap\mathcal{N}$.

If $\E^T_{\mathscr{A},\bm p} (\mathcal{R}\cap\mathcal{N}) \nsubseteq \mathcal{R}\cap\mathcal{N}$, pick $\tilde{\mathcal{N}}=\{\zero\}$ so that $\mathcal{R}\cap\tilde{\mathcal{N}}=\{0\}$ and Theorem \ref{thm:general_model_reduction} applies with ${\cal V} = \alg(\bm p\wedge \mathcal{R})$ .

Regarding ii) we have that, if $\mathscr{A}$ is unital, then Proposition \ref{prop:stochastic_factors} ensures that $J$ and $R$ are stochastic and thus $RPJ$ and $CJ$ are stochastic and $R\bm p_0$ is a probability vector for any $\bm p_0$ probability vector. If $\mathscr{A}$ is not unital, because of Corollary \ref{cor:stochastic_factors_non_unital} we have, that $J$ is stochastic (and thus $CJ$ is stochastic) but $R$ is only stochastic over $\supp(\mathscr{A})$, i.e. $\one_d^TR = \one^T_{\supp(\mathscr{A})}$. We next show that this condition is sufficient to show that the reduced model is stochastic. 

We shall first notice that $\supp(\mathcal{E}) = \supp(\mathscr{A})\subsetneq\mathbb{R}^n$. Let assume that $\dim(\supp(\mathcal{E}))=k$. Then we can consider a permutation (that is a double-stochastic change of basis) $T$ such that $T\bm x = \left[\begin{array}{c|c}\bm x'&\zero_{n-k}^T\end{array}\right]^T$ for all $\bm x\in\mathcal{E}$, with $\bm x'\in\mathbb{R}^k$. Then, since $\mathcal{E}$ is $P$-invariant 
$$ \left[\begin{array}{c}\bm x'\\\hline\zero_{n-k}\end{array}\right] = \underbrace{ \left[\begin{array}{c|c}P_{11}&P_{12}\\\hline P_{21}&P_{22}\end{array}\right]}_{TPT^T}\left[\begin{array}{c}\bm x'\\\hline\zero_{n-k}\end{array}\right] \in\mathcal{E}$$
and thus $P_{21} = 0$. This shows that $\supp(\mathcal{E})$ is $P$-invariant. Since $P$, $T$ and $T^T$ are stochastic, $TPT^T$ is also stochastic. This implies that $\one_k^TP_{11}=\one_k^T$. Then, it holds that $\one_{\supp(\mathscr{A})}^TT^TTPT^T =$ 
$$\left[\begin{array}{c|c}\one_k^T&\zero_{n-k}^T\end{array}\right] \left[\begin{array}{c|c}P_{11}&P_{12}\\\hline 0&P_{22}\end{array}\right] = 
\left[\begin{array}{c|c}\one_k^T&\zero_{n-k}^T\end{array}\right]  $$
or, in other words $\one_{\supp(\mathscr{A})}^TP =\one_{\supp(\mathscr{A})}^T$. We can also verify that $\check{P}$ is stochastic by verifying the following chain of equivalences: $\one_d^TRPJ=\one_{\supp(\mathscr{A})}^TPJ = \one_{\supp(\mathscr{A})}^TJ = \one_d^T$ where the last equality comes from Corollary \ref{cor:stochastic_factors_non_unital}. Finally, to prove that $R\bm p_0$ is a probability vector we can observe that $\one_n^T\bm p_0 = \one^T\Pi_{\mathcal{E}}\bm p_0 + \underbrace{\one^T\Pi_{\mathcal{R}\cap\mathcal{N}}\bm p_0}_{=\zero} = 1$ and then re-use the reasoning above.
\end{proof}

\begin{remark}
In the proof of Theorem \ref{thm:single_time_algo} we stated that a positive vector $\bm p$ is necessary to have a well-defined inner product $\inner{\cdot}{\cdot}_{\bm p}$. This assumption, however, can be relaxed to the following: $\bm p$ is positive over $\supp(\mathcal{E})=\supp(\mathscr{A})$, i.e. $\bm p_i>0$ for all $i$ such that $\bm e_i^T\bm x \neq 0 $ for some $\bm x\in\mathcal{E}$. This is due to the fact that the values of $\bm p$ where $\mathcal{E}$ has no support has no role in the projection. 

Although such $\bm p$ defines a positive semi-definite inner product over $\mathbb{R}^n$, it provides a positive definite inner product over $\supp(\mathcal{S})$ and this is sufficient to define the orthogonal projection onto $\mathscr{A}$. Consider, for example, the following case: assume $\supp(\mathscr{A})\subsetneq\mathbb{R}^n$ then let $\bm p_s$ be a positive vector over the $\supp(\mathscr{A})$, $\bm p_n$ be a positive vector over the remaining support, i.e. s.t. $\bm p:= \bm p_s +\bm p_n$, $\supp(\bm p) = \mathbb{R}^n$. We can then notice that $\bm p\wedge \bm x = \bm p_s \wedge \bm x$ and $\inner{\bm y}{\bm x}_{\bm p} = \inner{\bm y}{\bm x}_{\bm p_s}$ for all $\bm x \in \supp(\mathscr{A})$ and $\bm y\in\mathbb{R}^n$. This implies that 
$$ \E_{|\mathscr{A}, \bm p}^T = \sum_{j=1}^d \frac{(\bm p\wedge \bm a_j)\bm a_j^T}{\inner{\bm a_j}{\bm a_j}_{\bm p}} = \sum_{j=1}^d \frac{(\bm p_s\wedge \bm a_j)\bm a_j^T}{\inner{\bm a_j}{\bm a_j}_{\bm p_s}} = \E_{|\mathscr{A}, \bm p_s}^T.$$

The role of the positivity of $\bm p$ will be further discussed in Section \ref{sec:optimal_choice_of_parameters}.
\end{remark}

\section{Multi-time solution}
\label{sec:multi-time_solution}
The solution of Problem \ref{prb:multi_time} follows the same ideas presented in the previous section. In fact, the algorithm we propose to solve Problem \ref{prb:multi_time} is identical to the previous algorithm but for the involved subspaces. 
We now present our proposed method to solve Problem \ref{prb:multi_time}. This method takes the form of the following Algorithm, where $\Gamma$ is defined as in the previous section.

\begin{algorithm}
    \caption{HMM reduction for problem \ref{prb:multi_time}}
    \label{algo:model_reduction_multi_time}
    \SetAlgoLined
    \Input{$(P,C)$, $\mathcal{S}$.}
    \Parameters{$\bm p$, $\Gamma$.}
    Compute $\mathcal{R_C}$ and $\mathcal{N_C}$ using equations \eqref{eqn:cond_reachable_definition} and \eqref{eqn:cond_non_observable_definition}\;
    Compute $\mathcal{E_C} = \Gamma(\mathcal{R_C},\mathcal{N_C})$\;
    Compute $\mathscr{A}_{\mathcal{C}}=\alg({\bm p}^{-1}\wedge\mathcal{E_C})$\;
    Compute $\mathbb{E}_{|\mathscr{A}_{\mathcal{C}}, {\bm p}}^T$ using equation \eqref{eqn:dual_cond_prob_def}\;
    If  $\E^T_{|\mathscr{A}_\mathcal{C},\bm p}(\mathcal{R_C}\cap\mathcal{N_C})\nsubseteq\mathcal{R_C}\cap\mathcal{N_C}$:
    redefine $\mathscr{A}_\mathcal{C}:=\alg({\bm p}^{-1}\wedge\mathcal{R_C})$ and recompute $\mathbb{E}_{|\mathscr{A}_{\mathcal{C}}, {\bm p}}^T$\;
    Compute the factors $R$ and $J$ of $\mathbb{E}_{|\mathscr{A}_\mathcal{C}, {\bm p}}^T$ with the definition given in equation \eqref{eqn:exp_factorization}\;
    \Output{$(\check{P}, \check{C})=(RPJ,CJ)$ and $R$}
\end{algorithm}

We are finally ready to prove that Algorithm \ref{algo:model_reduction_multi_time} solves the multi-time marginal problem.

\begin{theorem}
\label{thm:multi_time_algo}
For any choice of $\mathcal{E_C}$ and $\bm p$ positive, i.e. ${\bm p_i}>0$ $\forall i$, Algorithm \ref{algo:model_reduction_multi_time} provides a solution to Problem \ref{prb:multi_time}.
\end{theorem}
\begin{proof}
The proof of this theorem follows the lines of the proof of Theorem \ref{thm:single_time_algo}. In fact, the proof of the fact that the reduced HMM $\check{\theta}$ is stochastic and $R\bm p_0$ is a probability vector is identical to the one given in \ref{algo:model_reduction}. 
The only difference in the two proofs regards proof of the fact that the reduced model $\check{\theta}$ with initial condition $R\bm p_0$ provides the same probability of a sequence of events as the model $\theta$ with initial condition $\bm p_0$. 

From Corollary \ref{cor:multi-time_model_reduction} we have that $(R\diag(\bm e_i^TC)PJ, \one^TJ)$ with initial condition $R\diag(\bm e_i^TC)\bm p_0$ generates the same probability as the original model. 
Since $\mathcal{R_C}$ and $\mathcal{N_C}$ are both $P$ and $\diag(\bm e_i^T C)$-invariant, Corollary \ref{cor:multi-time-model-reduction_appendix} applies, thus leading to the reduced HMM $\check{\theta} = (RPJ, CJ)$ and initial conditions $R\bm p_0$.
\end{proof}

\section{Choosing the algorithm's parameters}
\label{sec:optimal_choice_of_parameters}

In this section, we  discuss what is the best choice of the parameters for Algorithms \ref{algo:model_reduction} and \ref{algo:model_reduction_multi_time}. Being the structure of the two algorithms identical, we only discuss the optimal choice of $\mathcal{E}$ and $\bm p$: the results can be extended directly to $\mathcal{E}_\mathcal{C}$. The notion of optimality is related to the dimension of the reduced system, meaning: we want to find a choice of $\mathcal{E}$ and $\bm p$ positive such that the reduced model returned by Algorithm \ref{algo:model_reduction} has minimal dimension. This is equivalent to finding $\mathcal{E}$ and $\bm p$ such that $\alg(\bm p^{-1}\wedge\mathcal{E})$ has minimal dimension.

\subsection{Optimal distributions for observable HMMs}
We shall start the discussion by finding the optimal choice of $\bm p$ assuming that an effective subspace $\mathcal{E}$ is given. Before we prove the main result of this section, we shall first state the following useful result.
\begin{lemma}
    \label{lem:full_support}
    Given a vector space $\mathcal{W}\subseteq\mathbb{R}^n$ with generators $\{\bm w_i\}$, $\mathcal{W} = \Span\{\bm w_i\}$ there exists a vector $\bar{\bm w} := \sum_i \lambda_i \bm w_i$, with $\lambda_i\neq0$ for all $i$ and such that $\supp(\bar{\bm w}) = \supp(\mathcal{W})$.
\end{lemma}
The proof of this Lemma can be found in Appendix \ref{sec:proofs}.

\begin{theorem}
\label{thm:optimal_p}
  Let consider a vector space $\mathcal{W}\subseteq\mathbb{R}^n$ and a vector $\bar{\bm w}$ as in Lemma \ref{lem:full_support}. Then there exists a unique algebra $\mathscr{A}^*$ of minimal dimension such that $\mathcal{W}\subseteq\bm x\wedge\mathscr{A}^*$ for some  $\bm x\in\mathbb{R}^n.$
  Moreover, $\mathscr{A}^*=\alg(\bar{\bm w}^{-1}\wedge\mathcal{W})$ and it is unital over the support of $\mathcal{W}$, i.e. $\one_{\supp(\mathcal{W})}\in\mathscr{A}^*$.
\end{theorem}
\begin{proof}
    The existence of such a $\bar{\bm w}$ is proved in Lemma \ref{lem:full_support}.
    
    Since $\mathscr{A}=\mathbb{R}^n$ satisfies $\mathcal{W}\subseteq\bm x\wedge\mathscr{A}$, for all $\bm x\in\mathbb{R}^n$ and its possible sub-algebras are finite (corresponding to the partition of $n$), $\mathscr{A}^*$ exists. 
    To prove that it is an unique solution we  proceed by contradiction. Let assume that there exist two different algebras $\mathscr{A}, \mathscr{B}\subseteq\mathbb{R}^n$ with minimal dimension $\dim(\mathscr{A})=\dim(\mathscr{B})$ and two vectors $\bar{\bm a},\bar{\bm b}\in\mathbb{R}^n$ such that $\mathcal{W}\subseteq\bar{\bm a}\wedge\mathscr{A}$ and $\mathcal{W}\subseteq\bar{\bm b}\wedge\mathscr{B}$. From Proposition \ref{prop:alg} we know that $\mathscr{A} = \Span\{\bm a_j\}$ and $\mathscr{B} = \Span\{\bm b_j\}$ where $\{\bm a_i\}$ and $\{\bm b_i\}$ are the finest resolutions in $\idem(\mathscr{A})$ and $\idem(\mathscr{B})$ respectively. Clearly, if $\bm a_i = \bm b_i$ for all $i$ then $\mathscr{A}=\mathscr{B}$ which yields a contradiction. Therefore, we  assume that there exists an index $j$ such that $\bm a_j\neq\bm b_i$ for all $i$.
    We can then notice that for all $\bm v\in\mathcal{W}$, we can write $\bm v = \sum_i \mu_i \bar{\bm a}\wedge\bm a_i = \sum_i \nu_i \bar{\bm b}\wedge\bm b_i.$   
    
    For $j$ such that $\bm a_j\neq\bm b_i$ or all $i$ we can then write $$\bm a_j\wedge\bm v = \mu_j \bar{\bm a}\wedge\bm a_j = \sum_i \nu_i \bm a_j\wedge\bar{\bm b}\wedge \bm b_i.$$ The first equality implies that over the support of each $\bm a_j$ every $\bm v$ must be proportional to $\bar{\bm a}\wedge\bm a_j$. The second equality, on the other hand, due to the fact $\bm a_j\neq\bm b_j$ implies at least two of the products $\bm a_j\wedge\bar{\bm b}\wedge \bm b_i$ must be non-zero.
    In order for the nontrivial sum to be always proportional to $\bar{\bm a}\wedge\bm a_j$ it must be that the coefficients $\nu_i$ appear always in a fixed ratio.
    Hence, the corresponding $\bm b_i$ can be substituted by their sum, and still, generate the full $\cal W$ when multiplied by a suitable vector $\bar{\bm b}$. This shows that $\mathscr{B}$ could not be a minimal algebra unless $\bm a_i=\bm b_i$ for all $i,$ up to a reordering.
    
    Let then $\mathscr{A}^*$ be the unique algebra of minimal dimension such that $\mathcal{W}\subseteq{\bm x}\wedge\mathscr{A}^*$ for some $\bm w$. From Proposition \ref{prop:alg} we know that $\mathscr{A}^* = \Span\{\bm a_j\}$ where $\{\bm a_i\}$ is the finest resolution in $\idem(\mathscr{A}^*)$. In particular $\{\bm a_j\}$  forms an orthogonal basis for $\mathscr{A}^*$ and its elements have completing mutually-orthogonal supports, i.e. $\supp(\bm a_k)\perp\supp(\bm a_j)$ for $k\neq j$ and $\sum_j \bm a_j =\one_{\supp(\mathcal{W})}$.  We can then observe that $\bm x\wedge\mathscr{A}^* = \Span\{\bm x \wedge\bm a_j\}$ and that the vectors $\bm x\wedge \bm a_j$ have complementary mutually-orthogonal supports. Then for $\mathcal{W}\subseteq \bm x\wedge \mathscr{A}^*$ to hold it must be that $\bm w = \sum_j \mu_j \bm x\wedge\bm a_j$ for all $\bm w\in\mathcal{W}$. 
    
    By the above discussion we can write  $\bm w_i = \sum_j \mu_j^i \bm x\wedge \bm a_j$ for each generator of $\mathcal{W}$. Notice that, for all $j$, $\mu_j^i\neq 0$ for at least one $i$. Let then use the definition of $\bar{\bm w}$ given in the statement and, substituting the form of the $\bm w_i$ we just reported we obtain $\bar{\bm w} = \sum_j \sigma_j \bm x\wedge\bm a_j$ with $\sigma_j = \sum_i\lambda_i\mu_j^i$. From the argument above, from the fact that $\lambda_i\neq0$ for all $i$ and from the fact that, by hypothesis, $\bar{\bm w}$ has maximal support, we have that $\sigma_j\neq0$ for all $j$. Because of the  structure of $\{\bm x\wedge\bm a_j\}$ we have that $$(\bm a_j\wedge\bar{\bm w})^{-1} = \bm a_j\wedge\bar{\bm w}^{-1}=\sigma_j^{-1} (\bm x\wedge\bm a_j)^{-1} = \sigma_j^{-1} \bm x^{-1}\wedge\bm a_j,$$ and thus $\bar{\bm w}^{-1} = \sum_j \sigma_j^{-1} \bm x^{-1}\wedge \bm a_j.$
    From this we have that the vector space $\bar{\bm w}^{-1}\wedge\mathcal{W}$ is generated by vectors of the type \[\bar{\bm w}^{-1}\wedge\bm w_i=\sum_{j,k} \sigma_j^{-1} \mu_k^i\bm x^{-1}\wedge \bm a_j\wedge \bm x\wedge \bm a_k = \sum_{j} \sigma_j^{-1} \mu_j^i\bm a_j.\] This proves that $\bar{\bm w}^{-1}\wedge\mathcal{W}\subseteq\mathscr{A}^*$ and that any vector $\bm v\in\bar{\bm w}^{-1}\wedge\mathcal{W}$ can be written as $\bm v  = \sum_i v_i \bar{\bm w}^{-1}\wedge\bm w_i = \sum_j \xi_j \bm a_j$ with $\xi_j:=\sum_i v_i\sigma_j^{-1} \mu_j^i$. Let then consider any two vectors $\bm v, \bm u\in\bar{\bm w}^{-1}\wedge\mathcal{W}$ and compute their $\wedge$-product, \[\bm v\wedge\bm u = \left(\sum_j\xi_j\bm a_j\right)\wedge\left(\sum_j\hat{\xi}_j\bm a_j\right) = \sum_j \xi_j\hat{\xi}_j\bm a_j.\]  This implies that $\alg(\bar{\bm w}^{-1}\wedge\mathcal{W}) \subseteq \mathscr{A}^*$. On the other hand, it trivially holds that $\mathcal{W}\subseteq\bar{\bm w}\wedge \alg(\bar{\bm w}^{-1}\wedge\mathcal{W})$. But then, since we assumed that $\mathscr{A}^*$ was the unique algebra of minimal dimension such that $\mathcal{W}\subseteq{\bm x}\wedge\mathscr{A}^*$ for some $\bm x$ it must hold that $\alg(\bar{\bm w}^{-1}\wedge\mathcal{W}) = \mathscr{A}^*$.
    
    Finally, since $\bar{\bm w}\in\mathcal{W}$, then $\bar{\bm w}^{-1}\wedge\bar{\bm w} = \one_{\supp(\mathcal{W})}\in(\bar{\bm w}^{-1}\wedge\mathcal{W})\subseteq\mathscr{A}^*$.  
\end{proof}

\begin{remark}
Theorem \ref{thm:optimal_p} shows that, given any choice of the effective subspace, we can  construct a vector $\bar{\bm w}$ such that the algebra $\alg(\bar{\bm w}^{-1}\wedge\mathcal{E})$ has minimal dimension. However, not all such $\bar{\bm w}$ are positive over the support of $\mathcal{E}$. As a matter of fact, it could happen that some choices of $\mathcal{E}$ do not contain any non-negative vector, while $\bar{\bm w} = \bm p$ being non-negative is fundamental to construct a stochastic reduction. 
\end{remark}

Theorem \ref{thm:optimal_p}  is nonetheless sufficient to determine the optimal reduction for a class of HMMs, namely those for which $\cal R$ is ``observable'', i.e. ${\cal R}\cap {\cal N}=\emptyset .$
\begin{proposition}
    Let $\{\bm r_i\}$ be an $N$-dimensional set of positive generators of $\mathcal{R}$ and let $\bar{\bm p} := \sum \bm r_i/N$.
    Then, if ${\cal R}\cap {\cal N}=\{\zero\}$,   $\mathscr{A}:=\alg(\bar{\bm p}^{-1}\wedge\mathcal{R})$ provides the optimal reduction.
\end{proposition}
\begin{proof}
    By hypothesis we have ${\cal R}\cap {\cal N}=\emptyset .$ This implies that  $\mathcal{E}=\mathcal{R}$. Then, using Theorem \ref{thm:optimal_p} we have that $\bm p = \sum_i \bm r_i/N$, provides the minimal dimension for $\alg(\bm p^{-1}\wedge\mathcal{R})$ and thus the optimal reduction.
\end{proof}
Notice that this result applies in particular fully observable HMMs, i.e. when the pair $(P,C)$ is observable, and thus to finite-state Markov chains. In fact, the latter can be seen as HMMs with $C=I$. The corresponding optimal reduction is then a maximally-lumped version of the original process \cite{kemenyFiniteMarkovChains1983a}.

\subsection{Effective subspace for the general case}
In order to address the general case, in addition to a distribution $\bf p$  we also need to choose an effective subspace. Example 2 below illustrates that not all effective spaces are equivalent and lead to different dimensions for the reduced model, making this choice critical towards the optimality of the reduction. A natural candidate effective subspace is $\mathcal{E}_\perp$, the orthogonal complement (with respect to the natural inner product $\inner{{\bf x}}{{\bf y}}={\bf x}^T{\bf y}$) of $\mathcal{R}\cap\mathcal{N}$ in $\mathcal{R}$. 
Let then $\{\bm \varepsilon_i\}_{i=1,\ldots,d}$ be the set of generators of $\mathcal{E}_\perp.$ Then any choice of the effective subspace can be described as $\mathcal{E} = \Span\{\bm \varepsilon_i + \bm n_i\}_{i=1,\ldots,d},$ where $\{\bm n_i\}_{i=1,\ldots,d}$ is a set of vectors in $\mathcal{N}$.

We next show that the choice of the orthogonal complement $\mathcal{E}_\perp$ always allows for finding a {\em positive vector} $\bar{\bm w}=\bar{\bm p}$ as in the statement of Theorem \ref{thm:optimal_p}, and hence a valid stochastic reduction.  The following proposition is instrumental to this aim.
\begin{proposition}
Let $\bm p\in\mathbb{R}^n$ be a probability vector, and let $\mathcal{V}$ be a vector space such that $\one^T\bm v = 0$ for all $\bm v\in\mathcal{V}$. Let then  $\Pi_{\mathcal{V}}$ be the orthogonal projector on $\mathcal{V}$ with respect to the standard inner product $\inner{\cdot}{\cdot}$. Then $\bm q := \bm p - \Pi_{\mathcal{V}}\bm p$ is a probability vector.
\end{proposition}
\begin{proof}
    Let us start by defining $\bm w := \one/2 - \bm p$. We can then write $\bm p = \one/2-\bm w$ to notice that $\bm p_i\in[0,1]$ if and only if $-1/2\leq\bm w_i\leq 1/2$, that is if and only if $\norm{\bm w}_{\infty}\leq 1/2$. Moreover, we have that $\one^T\bm p=1$ if and only if $\one^T\bm w= (n-2)/2$.
    We can then compute $\bm q$:
    \begin{align*}
        \bm q &= \one/2 - \bm w - \Pi_\mathcal{V}\one/2 + \Pi_\mathcal{V}\bm w\\
        &= \one/2 - \underbrace{(I -\Pi_\mathcal{V})}_{=:\Pi_{\mathcal{V}^\perp}}\bm w= \one/2 - \Pi_{\mathcal{V}^\perp}\bm w
    \end{align*}
    where we used the the hypothesis $\one^T\bm v = 0$ for all $\bm v\in\mathcal{V}$ to say that $\Pi_\mathcal{V}\one = \zero$. Then, since $\Pi_{\mathcal{V}^\perp}$ is an orthogonal projection, and thus a contraction in norm, we have that $\norm{\Pi_{\mathcal{V}^\perp}\bm w}_\infty\leq\norm{\bm w}_\infty$. Then, using the argument above, we have that $\bm q$ is a non-negative vector with $\bm q_i\in[0,1]$.
    Lastly we have that $\one^T\bm q = \one^T\one/2 - \one^T\bm w = n/2 - (n-2)/2 = 1$.
\end{proof}

The result we are after is then obtained as a corollary of the previous one.
\begin{corollary}
\label{corollary:e_perp}
Let $\mathcal{E}_\perp$ be the orthogonal complement of $\mathcal{R}\cap\mathcal{N}$ to $\mathcal{R}$. Let $\{\bm r_i\}$ be an $N$ dimensional set of probability vectors such that $\mathcal{R}=\Span\{\bm r_i\}$. Then $\bm \varepsilon_i := \bm r_i - \Pi_{\mathcal{R}\cap\mathcal{N}}\bm r_i$  are such that $\mathcal{E}_\perp=\Span\{\bm \varepsilon_i\}$. Moreover, $\bar{\bm \varepsilon} = \sum_i\bm \varepsilon_i/N$ satisfies $\supp(\bar{\bm \varepsilon})=\supp(\mathcal{E})$ and $\bar{\bm \varepsilon}_i \geq0$ for all $i$.
\end{corollary}
\begin{proof}
    From Lemma \ref{lem:non_obs_orth_to_one} we have that $\one^T\bm x = 0$ for all $\bm x\in\mathcal{N}$ and thus, by applying the proposition above on every generator of $\mathcal{R}$ we have that the set $\{\bm\varepsilon_i \}$ is a set of probability vectors. Being $\bar{\bm \varepsilon}$ a convex combination of probability vectors it is itself a probability vector and it shares the same support as $\mathcal{E}$.
\end{proof}

Other choices are possible, and the choice of the  effective subspace can influence the dimension of the reduced model, as illustrated in the following example.

\begin{example}
Consider the following spaces:
\begin{align*}
    \mathcal{R}=\Span\left\{\begin{bmatrix}1/2\\1/2\\0\\0\end{bmatrix},\begin{bmatrix}0\\0\\1\\0\end{bmatrix},\begin{bmatrix}0\\0\\0\\1\end{bmatrix}\right\},\quad \mathcal{N}=\Span\left\{\begin{bmatrix}0\\0\\1\\-1\end{bmatrix}\right\}
\end{align*}
then we clearly have that $\mathcal{R}\cap\mathcal{N}=\mathcal{N}$. Let us denote with $\mathcal{E}_\perp$ the orthogonal complement of $\mathcal{R}\cap\mathcal{N}$ to $\mathcal{R}$, i.e. 
\begin{align*}
    \mathcal{E}_\perp=\Span\left\{\begin{bmatrix}1/2\\1/2\\0\\0\end{bmatrix},\begin{bmatrix}0\\0\\1/2\\1/2\end{bmatrix}\right\}.
\end{align*}
We can easily notice that $\mathcal{E}_\perp$ is an unital algebra. Let us now consider another completion $\mathcal{E}$ of $\mathcal{R}\cap\mathcal{N}$ to $\mathcal{R}$. In general, we can write 
\begin{align*}
    \mathcal{E}=\Span\left\{\begin{bmatrix}1\\1\\a\\-a\end{bmatrix},\begin{bmatrix}0\\0\\1+b\\1-b\end{bmatrix}\right\}
\end{align*}
for some values $a,b\in\mathbb{R}$. We can then consider two cases. First, if $a=0$ and $b\neq0$ we can choose $\bar{\bm v} = \begin{bmatrix}1&1&1+b&1-b\end{bmatrix}^T$ thus obtaining $\alg(\bar{\bm v}^{-1}\wedge\mathcal{E}) = \mathcal{E}_\perp$. On the other hand, if we have $a\neq0$ and $b\neq0$ we can choose $\bar{\bm w} = \begin{bmatrix}1&1&a+1+b&-a+1-b\end{bmatrix}^T$ (assuming that $a+b\neq\pm 1$) thus obtaining 
\begin{align*}
    \alg(\bar{\bm w}^{-1}\wedge\mathcal{E})=\Span\left\{\begin{bmatrix}1\\1\\0\\0\end{bmatrix},\begin{bmatrix}0\\0\\1\\0\end{bmatrix},\begin{bmatrix}0\\0\\0\\1\end{bmatrix}\right\}.
\end{align*}

This example shows that the choice of the effective subspace can affect the size of the reduced model. 
\end{example}

\section{Examples}
\label{sec:examples}

\begin{example} Let consider the HMM provided in \cite[Example 3]{itoIdentifiabilityHiddenMarkov1992}:
\begin{align*}
    P &= \left[\begin{array}{ccccc}
    1/3& 1/6& 1/4& 1/4&  0 \\
    1/6& 1/3&  0 & 1/4& 1/4\\
    1/3& 1/6& 1/4& 1/4&  0 \\
    1/6& 1/6& 1/6&  0 & 1/2\\
    0  & 1/6& 1/3& 1/4& 1/4
    \end{array}\right], \quad \mathcal{S} = \left\{\begin{bmatrix}
    1/5\\1/5\\1/5\\1/5\\1/5
    \end{bmatrix}\right\}\\
    C &= \left[\begin{array}{ccccc}
        1 & 1 &  1  &  0  &  0  \\
        0 & 0 &  0  &  1  &  0  \\
        0 & 0 &  0  &  0  &  1  \\
    \end{array}\right].
\end{align*}

We shall start by studying the single-time marginal problem. We can observe that $\bm p_0\in\mathcal{S}$ is an equilibrium for $P$ and thus $\mathcal{R} = \Span\{\mathcal{S}\}$ and also that $\mathcal{N}=\Span\{\begin{bmatrix}1&-2&1&0&0\end{bmatrix}^T\}$. Clearly, the intersection contains only the zero vector, $\mathcal{R}\cap\mathcal{N}=\{\zero\}$ and thus the effective subspace can be taken as the reachable one: $\mathcal{E}_\perp = \mathcal{R}$. If we then take $\bar{\bm p}=\bm p_0$ we obtain $\mathscr{A}=\alg(\bar{\bm p}^{-1}\wedge\mathcal{E}) = \Span\{\one\}$. The corresponding stochastic reduction and injection matrices are $R=\one^T$ and $J=\bar{\bm p}$ which provide the (trivial) reduced model:
\begin{align*}
    \check{P} &= \begin{bmatrix}1\end{bmatrix},\quad \check{\mathcal{S}}=\left\{\begin{bmatrix}1\end{bmatrix}\right\},\quad
    \check{C} = \begin{bmatrix}3/5&1/5&1/5\end{bmatrix}^T.
\end{align*}

We next focus on the multi-time marginal problem. We have that $\mathcal{N_C}=\mathcal{N},$ while the conditioned-reachable is equal to:
$$ \mathcal{R_C} = \Span\left\{
\begin{bmatrix}1/5\\1/5\\1/5\\0\\0\end{bmatrix},
\begin{bmatrix}0\\0\\0\\1/5\\0\end{bmatrix},
\begin{bmatrix}0\\0\\0\\0\\1/5\end{bmatrix},
\begin{bmatrix}1\\-2\\1\\0\\0\end{bmatrix}
\right\}.$$
This implies that the intersection $\mathcal{R_C}\cap\mathcal{N_C}=\mathcal{N_C}$ and thus:
$$ \mathcal{E_C}_\perp = \Span\left\{
\begin{bmatrix}1\\1\\1\\0\\0\end{bmatrix},
\begin{bmatrix}0\\0\\0\\1\\0\end{bmatrix},
\begin{bmatrix}0\\0\\0\\0\\1\end{bmatrix}
\right\}.$$
Then, we can notice that $\mathcal{E_C}$ is a unital algebra and by taking $\bar{\bm p}=\one/5$ we obtain  the stochastic reduction and injection matrices 
$$ R = \begin{bmatrix}1&1&1&0&0\\0&0&0&1&0\\0&0&0&0&1\end{bmatrix},\quad 
J = \begin{bmatrix}1/3&0&0\\1/3&0&0\\1/3&0&0\\0&1&0\\0&0&1\end{bmatrix}$$
that leads to the reduced model 
\begin{align*}
    \check{P} &= \begin{bmatrix}2/3&3/4&1/4\\1/6&0&1/2\\1/6&1/4&1/4\end{bmatrix},\quad \check{\mathcal{S}}=\left\{\begin{bmatrix}3/5\\1/5\\1/5\end{bmatrix}\right\}\\
    \check{C} &= \begin{bmatrix}1&0&0\\0&1&0\\0&0&1\end{bmatrix}.
\end{align*}
\end{example}

\begin{example}%[Augusto's Example]
\label{ex:augusto}
Consider the HMM defined by:
\begin{align*}
   P &= \left[\begin{array}{ccccc}
        1/2&  0 & 1/3& 1/4\\
         0 & 1/3& 1/3& 1/4\\
        1/2&  0 & 1/3&  0 \\
        0  & 2/3&  0 & 1/2
    \end{array}\right],\quad \mathcal{S}=\left\{\begin{bmatrix}1\\0\\0\\0\end{bmatrix},\begin{bmatrix}0\\1\\0\\0\end{bmatrix}\right\} \\
    C &= \left[\begin{array}{ccccc}
        1/4 & 1/4 &  1/2  &  7/16  \\
        3/4 & 3/4 &  1/2  &  9/16  
    \end{array}\right].
\end{align*}
In this case we are only interested in the single-time marginal problem. We can notice that $\mathcal{R}=\mathbb{R}^n$ and thus 
$$\mathcal{R}\cap\mathcal{N} = \mathcal{N} =\Span \left\{\begin{bmatrix}1\\-1\\0\\0\end{bmatrix},\begin{bmatrix}-1\\0\\-3\\4\end{bmatrix}\right\}.$$
Then we can consider the effective subspace as the orthogonal complement of $\mathcal{N}$, that is 
$$\mathcal{E}_\perp =\Span \left\{ \begin{bmatrix}4/9\\4/9\\0\\1/9\end{bmatrix}, \begin{bmatrix}0\\0\\4/7\\3/7\end{bmatrix}\right\}.$$
Then we can define $\bar{\bm p} := \one/4$ to obtain 
$$\mathscr{A} = \alg(\bar{\bm p}^{-1}\wedge\mathcal{E}_\perp) = \Span \left\{\begin{bmatrix}1\\1\\0\\0\end{bmatrix},\begin{bmatrix}0\\0\\1\\0\end{bmatrix}, \begin{bmatrix}0\\0\\0\\1\end{bmatrix}\right\}.$$
Notice that in this case, the dimension of the algebra is greater than the effective subspace.
We thus obtain  the stochastic reduction and injection matrices 
$$ R = \begin{bmatrix}1&1&0&0\\0&0&1&0\\0&0&0&1\end{bmatrix},\quad 
J = \begin{bmatrix}1/2&0&0\\1/2&0&0\\0&1&0\\0&0&1\end{bmatrix}$$
that leads to the reduced model 
\begin{align*}
    \check{P} &= \begin{bmatrix}5/12&2/3&1/2\\1/4&1/3&0\\1/3&0&1/2\end{bmatrix},\quad \check{\mathcal{S}}=\left\{\begin{bmatrix}1\\0\\0\end{bmatrix}\right\}\\
    \check{C} &= \begin{bmatrix}1/4&1/2&7/16\\3/4&1/2&9/16\end{bmatrix}.
\end{align*}

Suppose that, instead of the orthogonal complement, we were to consider the following space as an effective subspace:
$$\mathcal{E} =\Span \left\{ \begin{bmatrix}6\\5/2\\3/2\\-1\end{bmatrix}, \begin{bmatrix}2\\5/6\\25/2\\-25/3\end{bmatrix}\right\}.$$
We can immediately notice that there is no convex combination of the generators of $\mathcal{E}$ such that it is positive, however, if we consider $\bm v = \begin{bmatrix}8&10/3&14&-28/3\end{bmatrix}^T$ we have that 
$$\mathscr{A} = \alg({\bm v}^{-1}\wedge\mathcal{E}) = \Span \left\{\begin{bmatrix}1\\1\\0\\0\end{bmatrix},\begin{bmatrix}0\\0\\1\\1\end{bmatrix}\right\}$$ thus showing that a smaller algebra could be found for the reduction, if we were to consider vectors that are not non-negative.
\end{example}

\section{Conclusions and Future Work} 
\label{sec:conclusion}

In this work, we exploited system-theoretic ideas and algebraic representation of probability spaces to obtain effective reductions of HMMs that preserve the marginals of the original output process, in either the single- or multi-time case. 
While optimal reductions are explicitly characterized for a class of HMMs, including observable ones, the freedom of choice in the effective subspace makes finding the optimal reductions more challenging in the general case. Nonetheless, we provide an algorithm that produces reduced HMMs of minimal dimension in all considered examples. Based on the analytical and numerical examples we examined, we formulate the following conjecture on the optimality of the natural orthogonal complement.
\begin{conjecture}
Let $\mathcal{E}_\perp$ be defined as the (standard) orthogonal complement of ${\cal N}\cap {\cal R}$ to ${\cal R},$ and let $\bar{\bm p}$ be defined as in Corollary \ref{corollary:e_perp}. Then, given any other choice of $\mathcal{E}$ and $\bm w$ {\em non-negative} it holds that 
$$ \dim(\alg(\bar{\bm p}^{-1}\wedge\mathcal{E}_\perp)) \leq \dim(\alg({\bm w}^{-1}\wedge\mathcal{E})).$$
\end{conjecture}

\begin{remark} 
If the effective subspace is already an algebra with respect to a $\bar {\bm p}$ - inner product then $\dim(\mathcal{E}_\perp)=\dim(\alg(\bar{\bm p}^{-1}\wedge\mathcal{E}_\perp)),$ since  $\dim(\mathcal{E}^\perp) = \dim(\mathcal{E})$ and  $\dim(\mathcal{E})\leq \dim(\alg({\bm w}^{-1}\wedge\mathcal{E}))$ by Theorem \ref{thm:optimal_p} then the choice of $\mathcal{E}_\perp$ is optimal.
Also notice that removing the assumption that ${\bm w}$ is non-negative makes the statement false. A counterexample is presented at the end of Example \ref{ex:augusto}. However, having $\bm w$ non-negative is necessary in order to obtain a stochastic model.
\end{remark}
Proving the conjectured minimality may require novel mathematical ideas: the choice of ${\cal E},{\cal E_C}$ that minimize the size of the generated algebras is equivalent to identify the representative of the quotient space that can be described with the least number of indicator vectors, and a way to relate this notion to orthogonality to $\mathscr{N}$ does not seem straightforward to find.

Other natural developments of the proposed framework include a relaxation of the method so that it allows for {\em approximate preservation of the marginals}, thus yielding reductions in practical situations where noise and partial knowledge might make the exact equivalence we require in this work too stringent, due to the fact that controllable pairs are a dense set \cite{LM}. In addition, in many algorithms used to estimate HMMs from data, e.g. \cite{HSU20121460}, the dimension of the ``hidden'' state space (i.e. $n$) is assumed to be known. When this is not the case, one could estimate an HMM with a larger than necessary number of hidden variables, and then use an approximate reduction scheme to reduce the estimated model to one of more manageable size. Future work will also be devoted to the adaptation and application of the method to approximate coarse-graining of large-scale systems, to address otherwise untreatable problems \cite{antoulas2005approximation, cheng2021model,sandberg2009model}.

The algebraic approach also naturally extends to the non-commutative domain, and our method will be extended to quantum systems, in particular quantum walks and open systems in general. Analogies between HMM and quantum walks have been already noted in \cite{AST-PRA}, as well as \cite{5714268} and \cite{4895457}, which extend the result of \cite{itoIdentifiabilityHiddenMarkov1992} to include quantum walks.
Lastly, the algebraic viewpoint makes our results potentially interesting towards the solution of outstanding open problems in realization theory and model reduction for positive systems \cite{benvenuti2004tutorial}.

\section{Acknowledgements}
T.G. and F.T. wish to thank Lorenzo Finesso, Augusto Ferrante and Lorenza Viola for motivating and stimulating discussions on the topics of this work.

\printbibliography{}

\appendices

\section{A reduction result for switching autonomous systems}
\label{sec:switching}

This appendix is dedicated to introducing a general condition ensuring exact model reduction for switching autonomous systems. Both the single-time and the multi-time marginals can be described by the dynamics of this type. Consider a discrete-time, autonomous, switching, linear system  
\begin{equation*}
\begin{cases}
    \bm x(t+1) = F_i \bm x(t)\\
    \bm y(t) = H \bm x(t)\\
    \bm x(0) \in \mathcal{I}
\end{cases}
\end{equation*} denoted by the triplet $(\{F_i\},H,\mathcal{I})$. 
The evolution at any time clearly depends on the sequence of evolutions $F_i$ activated. 
Let us denote with $\bm y(s_{0:l})$ the output of the system  at time $l$ associated to a sequence $s_{0:l} =s_l,\dots,s_0$ of length $l$ of selected evolution $F_{s_k}$. The output at any time $l>0$ can be computed as \(\bm y(s_{0:l}) = H \prod_{j=l}^0 F_{s_j} \bm x_0\) while for $t=0$ we have \(\bm y(0) = H\bm x_0\). 

Let $\mathcal{R}\subseteq\mathbb{R}^n$ be a linear subspace such that $\mathcal{I}\subseteq\mathcal{R}$ and is $F_i$-invariant, i.e. $ F_i\mathcal{R}\subseteq\mathcal{R}$, for all $i$. Let $\tilde{\mathcal{N}}\subseteq\mathbb{R}^n$ be a linear subspace such that $\tilde{\mathcal{N}}\subseteq\ker H$ and is $F_i$-invariant, i.e. $F_i\tilde{\mathcal{N}}\subseteq\tilde{\mathcal{N}}$, for all $i$. Let then define  $\mathcal{E}$ to be any completion of $\mathcal{R}\cap\tilde{\mathcal{N}}$ to $\mathcal{R}$, i.e. $\mathcal{R} = (\mathcal{R}\cap\tilde{\mathcal{N}}) \oplus \mathcal{E}$.
\begin{theorem}
\label{thm:general_model_reduction}
Consider any subspace $\mathcal{V}$ such that $\mathcal{E}\subseteq\mathcal{V}$ with $m=\dim(\mathcal{V})$ and let $\Pi_\mathcal{V}$  be the orthogonal projection onto $\mathcal{V}$ with respect to an inner product $\inner{\cdot}{\cdot}.$ Assume that $\Pi_\mathcal{V}(\mathcal{R}\cap\tilde{\mathcal{N}})\subseteq\mathcal{R}\cap\tilde{\mathcal{N}}$, and let $R:\mathbb{R}^n\to\mathbb{R}^m$ and $J:\mathbb{R}^m\to\mathcal{V}$ be two factors of the orthogonal projection, $\Pi_\mathcal{V} = JR$. 

Let consider the reduced model $(\{\check{F}_i\}, \check{H}, \check{\mathcal{I}}) = (\{RF_iJ\},HJ,R\mathcal{I})$. Then the reduced model reproduces the same output as the original model, i.e. \[H\prod_{j=l}^0F_{{s}_j}\bm x_0 = \check{H}\prod_{j=l}^0\check{F}_{{s}_j}\check{\bm x}_0\] for any sequence $s_{0:l}$ and any initial condition $\bm x_0\in\Span\{\mathcal{I}\}$ and the relative $\check{\bm x}_0 = R\bm x_0$.
\end{theorem}
\bigequation{
    \begin{align}
        \left[\prod_{j=l}^0F_{{s}_j} - \Pi_\mathcal{V}\prod_{j=l}^0F_{{s}_j}\Pi_\mathcal{V}\right]\bm x_0 &= 
        \left[(\Pi_{\cal{V}}+\Pi_{\mathcal{T}})F_{{s}_l}(\Pi_{\cal{V}}+\Pi_{\mathcal{T}})\prod_{j=l-1}^0F_{{s}_j} - \Pi_\mathcal{V}\prod_{j=l}^0F_{{s}_j}\Pi_\mathcal{V}\right]\bm x_0 \nonumber \\
        &= \left[[
            \Pi_{\cal{V}}F_{{s}_l}\Pi_{\cal{V}} +
            \Pi_{\mathcal{T}}F_{{s}_l}\Pi_{\cal{V}} + 
            F_{{s}_l}\Pi_{\mathcal{T}}]
        \prod_{j=l-1}^0F_{{s}_j} 
        - \Pi_{\cal{V}}F_{{s}_l}\Pi_{\cal{V}}\Pi_\mathcal{V}\prod_{j=l-1}^0F_{{s}_j}\Pi_\mathcal{V}\right]\bm x_0 \nonumber \\
        &= 
            \Pi_{\cal{V}}F_{{s}_l}\Pi_{\cal{V}} 
            \underbrace{\left(\prod_{j=l-1}^0F_{{s}_j} - \Pi_\mathcal{V}\prod_{j=l-1}^0F_{{s}_j}\Pi_\mathcal{V}\right)\bm x_0}_{\bm v}  
            +\left[[\Pi_{\mathcal{T}}F_{{s}_l}\Pi_{\cal{V}} + 
            F_{{s}_l}\Pi_{\mathcal{T}}]
        \prod_{j=l-1}^0F_{{s}_j} \right]\bm x_0 \nonumber \\
        &= \Pi_{\cal{V}}F_{{s}_l}\Pi_{\cal{V}} 
            \bm v 
            +\Pi_{\mathcal{T}}F_{{s}_l}\Pi_{\cal{V}}  \prod_{j=l-1}^0F_{{s}_j} \bm x_0
            + F_{{s}_l}\Pi_{\mathcal{T}}
        \prod_{j=l-1}^0F_{{s}_j} \bm x_0 
        \label{eq:big_eq}
    \end{align}
}
\begin{proof}%[Theorem \ref{thm:multi-time_model_reduction}]
Let $\mathcal{W}_1$ be the completion of $\mathcal{R}\cap\tilde{\mathcal{N}}$ to $\tilde{\mathcal{N}}$, i.e. $\tilde{\mathcal{N}} = \mathcal{W}_1 \oplus (\mathcal{R}\cap\tilde{\mathcal{N}})$; let $\mathcal{W}_2$ be the completion of $(\mathcal{R}\cap\tilde{\mathcal{N}})\oplus\mathcal{E}\oplus\mathcal{W}_1$ to $\mathbb{R}^n$, i.e. $\mathbb{R}^n = \mathcal{W}_1 \oplus \mathcal{W}_2 \oplus \mathcal{E} \oplus  (\mathcal{R}\cap\tilde{\mathcal{N}})$;  let $\mathcal{T}$ the remainder sub-space, such that $\mathbb{R}^n = \cal V \oplus \mathcal{T}$ and thus $\mathcal{T}\subseteq (\tilde{\mathcal{N}}\cap \mathcal{R})  \oplus \mathcal{W}_1 \oplus \mathcal{W}_2$. Let us also denote with $\Pi_\mathcal{T}$ the orthogonal projector onto $\mathcal{T}$ with respect to the considered inner product $\inner{\cdot}{\cdot}$. 

We can notice that, for any sequence $s_{0:l}$ we have that 
$\check{H}\prod_{j=l}^0\check{F}_{{s}_j}\check{\bm x}_0 = H \Pi_\mathcal{V}\prod_{j=l}^0F_{{s}_j}\Pi_\mathcal{V}\bm x_0$ and thus the statement can be also be put in the form 
\[H\left[\prod_{j=l}^0F_{{s}_j} - \Pi_\mathcal{V}\prod_{j=l}^0F_{{s}_j}\Pi_\mathcal{V}\right]\bm x_0  = 0\]
for any sequence $s_{0:l}$ and for any $\bm x_0\in\mathcal{I}$. To prove the statement we will thus show that for any sequence $s_{0:l}$ and for any initial condition $\bm x_0\in\mathcal{I}$ it holds \[\left[\prod_{j=l}^0F_{{s}_j} - \Pi_\mathcal{V}\prod_{j=l}^0F_{{s}_j}\Pi_\mathcal{V}\right]\bm x_0 \in\tilde{\mathcal{N}}\cap\mathcal{R}.\]
We will prove this statement by induction. 

Let then consider the case of $t=0$. We have to prove \([I-\Pi_{\cal V}]\bm x_0\in\tilde{\mathcal{N}}\cap\mathcal{R}\). Then by noticing that, $(I-\Pi_{\cal V})\bm x_0 = \Pi_{\mathcal{T}} \bm x_0$ and that $\Pi_{\mathcal{T}}\bm x_0\in \tilde{\mathcal{N}}\cap\mathcal{R}$, the statement is proved in the case $t=0$. 

Assume then that
\[\bm v := \left[\prod_{j=l-1}^0F_{{s}_j} - \Pi_\mathcal{V}\prod_{j=l-1}^0F_{{s}_j}\Pi_\mathcal{V}\right]\bm x_0 \in\tilde{\mathcal{N}}\cap\mathcal{R}\]
and we want to prove that \[\left[\prod_{j=l}^0F_{{s}_j} - \Pi_\mathcal{V}\prod_{j=l}^0F_{{s}_j}\Pi_\mathcal{V}\right]\bm x_0 \in\tilde{\mathcal{N}}\cap\mathcal{R}.\]
By rewriting this as in Equation \eqref{eq:big_eq} we can observe that it is equal to the sum of three parts. We can then notice that:
\begin{itemize}
    \item $\bm v \in\tilde{\mathcal{N}}\cap\mathcal{R}$, $\Pi_{\cal V}\bm v\in\tilde{\mathcal{N}}\cap\mathcal{R}$ by assumption, thus, $P\Pi_{\cal V}\bm v\in\tilde{\mathcal{N}}\cap\mathcal{R}$ and also $\Pi_{\cal V}P\Pi_{\cal V}\bm v\in\tilde{\mathcal{N}}\cap\mathcal{R}$;
    \item $\prod_{j=l-1}^0F_{{s}_j} \bm x_0\in\mathcal{R}$ by hypothesis, $\Pi_{\mathcal{T}}\prod_{j=l-1}^0F_{{s}_j} \bm x_0\in\tilde{\mathcal{N}}\cap \mathcal{R}$ and $F_{{s}_l}\Pi_{\mathcal{T}}\prod_{j=l-1}^0F_{{s}_j} \bm x_0\in\tilde{\mathcal{N}}\cap \mathcal{R}$;
    \item $\prod_{j=l-1}^0F_{{s}_j} \bm x_0\in\mathcal{R}$ by hypothesis, $\Pi_{\mathcal{V}}\prod_{j=l-1}^0F_{{s}_j} \bm x_0\in\mathcal{R}$, $F_{{s}_l}\Pi_{\cal V} \prod_{j=l-1}^0F_{{s}_j} \bm x_0 \in \cal R$ and $\Pi_{\mathcal{R}}F_{{s}_j}\Pi_{\cal V}\prod_{j=l-1}^0F_{{s}_j} \bm x_0\in \tilde{\mathcal{N}}\cap \cal R$.
\end{itemize}   
Finally, since all three summands belong to $\tilde{\mathcal{N}}\cap\mathcal{R}$, their sum also belongs to the same subspace, and the statement is proved. 
\end{proof}
In order to apply the result to our multi-time problem, we need a straightforward extension.
\begin{corollary}
    \label{cor:multi-time-model-reduction_appendix}
    Under the assumptions of Theorem \ref{thm:general_model_reduction}, let us further assume that $F_i$ are factorized as $F_i = D_iA$, that $\mathcal{R}$ and $\tilde{\mathcal{N}}$ are $A$-invariant and $D_i$-invariant for all $i$ and also that $\mathcal{I} = \bigcup_iD_i\mathcal{S}$ for some set $\mathcal{S}$. 
    
    Then the matrices $\{\check{F}_i\}$ of the reduced model can be taken to be $\check F_i = \check D_i \check A$ with $\check D_i= RD_iJ$, $\check A= RAJ.$ .
\end{corollary}
The proof of this corollary follows exactly that of  Theorem \ref{thm:general_model_reduction}, where  $ H \Pi_\mathcal{V}\prod_{j=l}^0(\Pi_\mathcal{V}D_{{s}_j} A\Pi_\mathcal{V})\Pi_\mathcal{V} D_i\bm x_0$ is substituted by $ H \Pi_\mathcal{V}\prod_{j=l}^0(\Pi_\mathcal{V}D_{{s}_j}\Pi_\mathcal{V} A\Pi_\mathcal{V})\Pi_\mathcal{V} D_i \Pi_\mathcal{V} \bm x_0$, and  in the induction we leverage the fact that $\tilde{\mathcal{N}},\mathcal{R}$ and thus $\tilde{\mathcal{N}}\cap\mathcal{R}$ are invariant for $\Pi_\mathcal{V}$, $A$ and $D_i,$ for all $i.$

\section{Proofs}
\label{sec:proofs}
This Appendix collects some proofs that were not included in the main text to improve readability.
\begin{proof}[Proof of Proposition \ref{prop:alg}]
    Let start with the first part of the statement. 
    The fact that $\mathscr{A}$ is closed under linear combinations and $\one\in\mathscr{A}$ follows directly from the definition of $\mathscr{A}$. The closure of $\mathscr{A}$ under element-wise product follows from the closure of $\mathcal{F}$ under the same operation. In particular let consider $\bm x,\bm y\in\mathscr{A}$, then $\bm x\wedge \bm y = \sum_{i,j} x_i y_j \bm f_i\wedge \bm f_j$ and , since $\bm f_i\in\mathcal{F}$ for all $i$, $\bm f_i\wedge \bm f_j\in\mathcal{F}$ and thus $\bm x\wedge \bm y\in\mathscr{A}$.  So $\mathscr{A}$ is an algebra, namely the set of $\cal F$-measurable random variables, and it is the minimal one by construction.
    
    We can then consider the second part of the statement.
    First of all, notice that the vectors that are idempotent for the element-wise product are composed only of zeros and ones.
    Let then consider a general element $\bm x\in\mathscr{A}$ and let $x_{i^*} = \max_{i=1,\dots,n}|x_i|$. We can then compute $\bm x' = \bm x/x_{i^*} \in\mathscr{A}$ that will have value 1 in the positions where $\bm x$ has value $x_{i^*}$, possibly values -1 in the position where $\bm x$ has value $-x_{i^*}$ and values in the range $(-1,1)$ in all the others positions. We can then define $\bm x'' = 0.5(\bm x'+\bm x'\wedge \bm x')\in\mathscr{A}$ that will have have value 1 in the positions where $\bm x$ has value $x_{i^*}$ and values in the range $(-1,1)$ in all the others positions. Finally the first idempotent element of the desired set is $\bm f_1 = \lim_{n\to \infty}(\bm x'')^n \in\mathscr{A}$ with element-wise power. Notice that $\bm f_1$ will have 1 in the same positions as $\bm x'$ and zeros in all the others. This implies that $\bm f_1$ is idempotent. By iterating the procedure on $\bm x - x_{i^*}\bm f_1,$ and so on, we obtain the whole set of idempotent elements $\{\bm f_i\}\subset\mathscr{A}$ such that $\bm x=\sum_i x_i \bm f_i$ up to a reordering of the coefficients $x_i$. We shall denote with $\idem(\bm x)$ the function that, given an element $\bm x$, returns the set of idempotent elements $\{\bm f_i\}$ that generate $\bm x$. We then have that $\idem(\mathscr{A}) \supseteq \cup_{\bm x\in\mathscr{A}}\idem(\bm x)$ by definition, while to prove $\idem(\mathscr{A}) \subseteq \cup_{\bm x\in\mathscr{A}}\idem(\bm x)$ it suffice to notice that each element of $\idem(\mathscr{A})$ is also an element of $\mathscr{A}$. This implies $ \idem(\mathscr{A}) = \cup_{\bm x\in\mathscr{A}}\idem(\bm x) $. Then, by construction, it holds that $\Span\{\idem(\mathscr{A})\} = \mathscr{A}$.
    
    We shall then notice that $\mathcal{F} = \idem(\mathscr{A})$ contains the elements, $\zero, \one\in\mathcal{A}$, and is closed under the operations $\wedge$, $\vee$ and $\neg$. This shows that $\idem(\mathscr{A})$ is a $\sigma$-algebra. Then, since $\mathscr{A}=\Span\{\idem(\mathscr{A})\}$ then any element in $\mathscr{A}$ is $\idem(\mathscr{A})$-measurable. Moreover, $\idem(\mathscr{A})$ is minimal because subtracting any element from it would make that element (seen as a r.v.) non-measurable. We thus have $\mathcal{F} = \idem(\mathscr{A})$.
    
    Finally, $\res(\mathcal{F})\subset\mathcal{F}$ is such that $\bm f_i\wedge \bm f_j = \zero$, for all $\bm f_i, \bm f_j \in\res(\mathcal{F})$ $i\neq j$. This implies that $\inner{\bm f_i}{\bm f_j} = \delta_{i-j}$ which means that is a set of orthogonal vectors. Moreover $\bm f = \vee_{\bm f_j\in S}\bm f_j$ with $S\subseteq\res(\mathcal{F})$ for all $\bm f\in\mathcal{F}\setminus\{\zero\}$ or, equivalently, $\bm f = \sum_j c_j\bm f_j$ with $c_j\in\{0,1\}$ for all $\bm f\in\mathcal{F}$. 
    This implies that $\mathscr{A} = \Span\{\res(\mathcal{A})\}$ and  thus $\res(\mathscr{A})$ is an orthogonal basis for $\mathscr{A}$ and $\dim(\mathscr{A}) = |\res(\mathscr{A})|$.
\end{proof}

\begin{proof}[Proof of Lemma \ref{lem:cond_exp_invariance}]
    First of all, note that the modified inner product can be written in many equivalent forms: $\inner{\bm v}{\bm w}_{\bm p} = \Ea_{\bm p}[\bm v\wedge \bm w] = \inner{\bm p}{\bm v\wedge\bm w} = \inner{\bm p\wedge\bm v}{\bm w} = \inner{\bm v}{\bm p\wedge\bm w}$.
    
    Let us then consider $\Ea_{|\mathscr{A},\bm p}$. We can notice that ${\rm image}(\Ea_{|\mathscr{A},\bm p}) = \mathscr{A}$ and that $\Ea_{|\mathscr{A},\bm p}^2=\Ea_{|\mathscr{A},\bm p}$. It remains to be proven the fact that $\Ea_{|\mathscr{A},\bm p}$ is self adjoint with respect to the inner product $\inner{\cdot}{\cdot}_{\bm p}$, that is $\inner{\bm v}{\Ea_{|\mathscr{A},\bm p}\bm w}_{\bm p} =\inner{\Ea_{|\mathscr{A},\bm p}\bm w}{\bm v}_{\bm p}$. 
    Such an equality can be rewritten, using equivalent forms of the modified inner product above, as $\inner{\bm v}{\bm p\wedge\bm p\wedge\Ea_{|\mathscr{A},\bm p}\bm w} = \inner{\bm p\wedge\bm p\wedge\Ea_{|\mathscr{A},\bm p}\bm v}{\bm w}$.
    That is equivalent to prove that $\bm p\wedge\Ea_{|\mathscr{A},\bm p}$ is self-adjoint with respect to the standard inner product, which can be verified by simply computing it.
    
    Identical reasoning can be done for $\Ea_{|\mathscr{A},\bm p}^T$. Note that ${\rm image}(\Ea_{|\mathscr{A},\bm p}^T) = \bm p\wedge\mathscr{A}$, that $(\Ea_{|\mathscr{A},\bm p}^T)^2=\Ea_{|\mathscr{A},\bm p}^T$ and that $\bm p^{-1}\wedge\Ea_{|\mathscr{A},\bm p}$ is self adjoint with respect to the standard inner product and the statement is proved. 
\end{proof}

\begin{proof}[Proof of Proposition \ref{prop:spaces_inclusions}]
Both $\ker C\supseteq \mathcal{N}$ and $\mathcal{S} \subseteq \mathcal{R}$  are well-known properties. We have to prove that $\mathcal{N} \supseteq \mathcal{N_C}$ and $\mathcal{R} \subseteq \mathcal{R_C}$. 

Regarding the reachable space we have that $\Span\{P^{y_{0:l}}\bm p_0, \forall {y_{0:l}} \text{ s.t. } l=k, \forall \bm p_0\in\mathcal{S}\} \supseteq  \Span\{P^{k}\bm p_0, \forall \bm p_0\in\mathcal{S}\}$ for all $k\geq0$. This is proven directly using lemma \ref{lem:str_sum}.

For the non-observable subspace it holds that $$ \begin{bmatrix}\one^T\diag(\bm e_0 )PP_C^{y_{0:l-1}}\\ \vdots \\ \one^T\diag(\bm e_m )PP_C^{y_{0:l-1}}
\end{bmatrix}= CPP_C^{y_{0:l-1}}\bm p_0.$$ Then, we have that $\ker[CPP_C^{y_{0:l-1}}] = \ker[CP^{l}]$ for all $y_{0:l-1}$ of length $l$, for any length $l$. Once again this is proved by using Lemma \ref{lem:str_sum}. Consider a vector $\bm v\in\ker[CPP_C^{y_{0:l-1}}]$ for all ${y_{0:l-1}}$. Then it holds that $CPP_C^{y_{0:l-1}}\bm v=0$ summing both sides of this equation over all sequences ${y_{0:l-1}}$ of length  $l-1$ and using the Lemma above we obtain $CP^{|{y_{0:l-1}}|+1}\bm v=0$, thus proving the statement.

The statement on the effective subspaces follows directly from the other two.
\end{proof}

\begin{proof}[Proof of Lemma \ref{lem:full_support}]
    We shall start by constructing a vector $\widetilde{\bm w}$ such that $\supp(\widetilde{\bm w}) = \supp(\mathcal{W})$. Starting from it we then construct a vector $\bar{\bm w}$ such that it is a linear combination of every generator.
    
    By definition of support of a vector space, for each $\bm e_i\in\supp(\mathcal{W})$ there exists a vector $\bm x_i\in\mathcal{W}$ such that $\bm e_i^T\bm x_i\neq 0$, forming a set $\{\bm x_i\}$. Without loss of generality, we assume ${i=0,\dots, m}$ with $m=\dim(\supp(\mathcal{W}))$. We can then define $\widetilde{\bm w}_0 = \bm x_0$ and iteratively compute $\widetilde{\bm w}_i = \widetilde{\bm w}_{i-1} + \lambda_i \bm x_i$ with $\lambda_i \notin \{-\bm e_j^T\widetilde{\bm w}_{i-1} / \bm e_j^T\bm x_i, \forall j| \bm e_j^T\bm x_i\neq 0\}\cup\{0\}$. Since this set is finite, it is always possible to choose a suitable $\lambda_i\in\mathbb{R}$ for each $i$. At the end of the iteration process, we obtain $\widetilde{\bm w}=\widetilde{\bm w}_m = \sum_{i=0}^m \lambda_i \bm x_i \in\mathcal{W}$. To prove that $\supp(\widetilde{\bm w})=\supp(\mathcal{W})$ we can simply observe that $\bm e_j^T\widetilde{\bm w} = \sum_{i=0}^m \lambda_i \bm e_j^T\bm x_i \neq 0$ by construction for all $\bm e_j\in\supp(\mathcal{W})$. On the other hand, for every $\bm e_j\notin\supp(\mathcal{W})$,  $\bm e_j^T\bm x_i=0$ for all $i$ and thus $\bm e_j^T\widetilde{\bm w} = 0$.
    
    This $\widetilde{\bm w}$ must be described as a linear combination of some of the generators, say $\widetilde{\bm w} = \sum_{i\in S}\lambda_i \bm w_i$, for some set of indices $S$. We can then use the same procedure as before: take $\bm w_i$ such that $i\notin S$, by choosing any $\lambda_i \notin \{-\bm e_j^T\widetilde{\bm w} / \bm e_j^T\bm w_i, \forall j| \bm e_j^T\bm w_i\neq 0\}\cup\{0\}$ we have $\supp(\widetilde{\bm w} + \lambda_i\bm w_i) = \supp(\mathcal{W})$. Iterating this procedure on the remaining vectors $\{\bm w_i|i\notin S\}$ we obtain $\bar{\bm w} = \widetilde{\bm w} + \sum_{i\notin S} \lambda_i\bm w_i$ such that $\supp(\bar{\bm w}) := \supp(\mathcal{W})$. 
\end{proof}

\end{document}